\documentclass[11pt]{article}

\usepackage{amsmath, amssymb, mathtools, amsthm, graphicx, float, bm, csquotes}
\usepackage{csquotes}
\usepackage{setspace}
\usepackage[margin=1in]{geometry}
\usepackage{makecell}
\usepackage{tabularx}
\usepackage{graphicx}
\usepackage{adjustbox}
\usepackage{verbatim}
\usepackage{float}
\usepackage{etoolbox}
\usepackage[flushleft]{threeparttable}
\usepackage{multirow}
\usepackage{booktabs}
\usepackage{algorithm2e}
\usepackage{algpseudocode}
\usepackage{breqn}
\usepackage{microtype}
\usepackage{xcolor}
\usepackage{tikz}
\usepackage{makecell}
\usepackage{bbm}
\usepackage{dsfont}
\usepackage{mdframed}
\usepackage{xr-hyper}
\usepackage[colorlinks=true,linkcolor=blue,citecolor=blue, urlcolor=blue]{hyperref}
\usepackage[nameinlink]{cleveref}
\usepackage{subcaption}
\usepackage{enumitem}

\usetikzlibrary{arrows.meta}

\newtheorem{proposition}{Proposition}
\newtheorem*{corollary}{Corollary}

\newtheorem*{proposition*}{Proposition}

\SetCommentSty{mycommfont}
\makeatletter
\renewcommand{\@algocf@capt@plain}{above}
\makeatother

\usepackage[backend=biber,style=authoryear, sorting=nyt, citestyle=authoryear, doi=false,isbn=false,url=false, maxcitenames=2]{biblatex}

\numberwithin{equation}{section}

\newcommand{\jobResumeSimCoef}{0.003}
\newcommand{\jobResumeSimDiff}{0.51}
\newcommand{\jobResumeSimPval}{$<0.001$}
\newcommand{\yrsExpCoef}{-0.53}
\newcommand{\yrsExpDiff}{-6.1}
\newcommand{\yrsExpPval}{$<0.001$}
\newcommand{\topschoolCoef}{0.013}
\newcommand{\topschoolDiff}{4.66}
\newcommand{\topschoolPval}{$<0.001$}
\newcommand{\bachelorsCoef}{0.015}
\newcommand{\bachelorsDiff}{1.79}
\newcommand{\bachelorsPval}{$<0.001$}
\newcommand{\mastersCoef}{0.03}
\newcommand{\mastersDiff}{7.16}
\newcommand{\mastersPval}{$<0.001$}
\newcommand{\doctorateCoef}{-0.004}
\newcommand{\doctorateDiff}{-6.98}
\newcommand{\doctoratePval}{$<0.001$}

\begin{document}

\begin{titlepage}
    \title{Algorithmic Hiring and Diversity: Reducing Human-Algorithm Similarity for Better Outcomes}

    \author{Prasanna Parasurama \\ \small Emory University
    \and 
    Panos Ipeirotis \\ \small New York University}
    \date{\today
    }
    \maketitle

    \begin{abstract}
        \noindent
    Algorithmic tools are increasingly used in hiring to improve fairness and diversity, often by enforcing constraints such as gender-balanced candidate shortlists. 
    However, we show theoretically and empirically that enforcing equal representation at the shortlist stage does not necessarily translate into more diverse final hires, even when there is no gender bias in the hiring stage. 
    We identify a crucial factor influencing this outcome: the correlation between the algorithm's screening criteria and the human hiring manager's evaluation criteria---higher correlation leads to lower diversity in final hires. 
    Using a large-scale empirical analysis of nearly 800,000 job applications across multiple technology firms, we find that enforcing equal shortlists yields limited improvements in hire diversity when the algorithmic screening closely mirrors the hiring manager's preferences. 
    We propose a complementary algorithmic approach designed explicitly to diversify shortlists by selecting candidates likely to be overlooked by managers, yet still competitive according to their evaluation criteria. 
    Empirical simulations show that this approach significantly enhances gender diversity in final hires without substantially compromising hire quality. 
    These findings highlight the importance of algorithmic design choices in achieving organizational diversity goals and provide actionable guidance for practitioners implementing fairness-oriented hiring algorithms.

        \bigskip
        \setcounter{page}{0}
        \thispagestyle{empty}

    \end{abstract}

\end{titlepage}

\pagebreak \newpage
\doublespacing
\section{Introduction}

In recent years, organizations have widely adopted various policies to increase workforce diversity~\parencite{shi_adoption_2018}. 
A popular policy is to diversify candidate shortlists or interview pools---often referred to as \textit{soft} affirmative action policies.
Unlike \textit{hard} affirmative action policies such as hiring quotas, which are explicitly prohibited by US employment law, soft policies aim only to increase minority representation in the initial interview stage without imposing quotas on final hires (Civil Rights Act of 1974; \cite{schuck_affirmative_2002}). 
Prominent examples of soft affirmative action policies include the NFL's Rooney Rule~\parencite{rooney_rule_nfl}, which requires interviewing at least one ethnic minority candidate for head coaching positions, and similar policies adopted by major tech firms such as \textcite{facebook_diversity}, ~\textcite{pinterest_diversity}, and ~\textcite{patreon_diversity}.

With the increasing use of algorithms in hiring, these diversity policies are frequently implemented as algorithmic fairness constraints. 
For instance, \emph{LinkedIn Recruiter} has deployed fairness-aware ranking algorithms aimed at improving gender diversity among candidates presented to recruiters~\parencite{geyik_fairness-aware_2019}. 
However, presenting more diverse candidate sets does not necessarily translate into greater diversity in hiring outcomes. 
LinkedIn's own analyses highlight uncertainty regarding whether improved gender representation in candidate recommendations leads to measurable improvements in outcomes, such as candidate contacts or interview requests~\parencite{geyik_fairness-aware_2019}.
Ultimately, algorithmic recommendations are integrated with human decisions, and the final hiring outcomes depend on human managers or recruiters.

Previous laboratory studies have indicated that the effectiveness of fairness constraints can vary significantly across job types~\parencite{suhr_does_2021,peng_what_2019}.
When these policies fail, conventional wisdom typically attributes their ineffectiveness to human biases. 
While human biases undoubtedly affect outcomes, other important factors influencing fairness constraints' effectiveness remain underexplored.

To systematically explore these factors, we propose and analyze a two-stage hiring model comprising algorithmic screening followed by human hiring decisions. 
The model considers a hiring scenario with a higher number of male applicants than female applicants, reflecting conditions typical in firms using diversity policies.
In the first stage, a screening algorithm shortlists candidates and applies an \textit{equal selection} constraint, such that an equal number of men and women are shortlisted. 
A hiring manager then evaluates the shortlisted candidates and hires the best candidates based on her own assessments.

Analytically solving this model reveals a crucial insight: the effectiveness of the equal selection constraint diminishes as the correlation between the screening algorithm's evaluation criteria and the hiring manager's evaluation criteria increases.
Moreover, the expected quality of hires also decreases as the correlation increases.
In other words, the better the screening algorithm matches the manager's preferences, the lower the expected quality of hires \emph{and} the less effective the equal selection constraint becomes. 
Based on this insight, we propose a complementary algorithm designed explicitly to select candidates likely to be overlooked by hiring managers yet still be competitive according to their evaluation criteria.

We empirically validate our theoretical predictions on hire diversity using extensive hiring data from eight technology firms, including nearly 800,000 applicants and over 3,600 job postings. 
Through counterfactual simulations, we demonstrate two key findings:
\begin{enumerate}
\item Consistent with our theoretical predictions, enforcing equal selection constraints in the shortlist does not consistently improve hiring diversity and may have negligible effects in some scenarios.
\item The constraint's effectiveness varies substantially across job types, driven primarily by differences in the correlation between algorithmic screening and managerial assessment criteria.
\end{enumerate}
Furthermore, when we benchmark our complementary algorithm against other traditional fairness constraints, we find it substantially more effective in improving workforce diversity without significant trade-offs in candidate quality.

\textbf{Our Contributions.} We theoretically and empirically show that the equal selection constraint, a common algorithmic fairness constraint, fails to increase workforce diversity when algorithmic screening evaluations correlate strongly with human hiring evaluations. 
To address this, we introduce and validate a complementary screening algorithm designed specifically to reduce these correlations, significantly improving hire diversity outcomes with minimal loss in candidate quality across various hiring contexts.
Our study contributes to the literature on algorithmic fairness in hiring pipelines in two key ways. 
First, we provide a theoretical characterization of when equal shortlist constraints effectively enhance diversity, emphasizing the critical role of correlation between screening and hiring evaluations. 
Second, we empirically validate this theoretical insight and introduce a complementary algorithmic design that significantly improves diversity outcomes in practice.

The remainder of the paper proceeds as follows. \Cref{sec:relatedwork} reviews related literature. \Cref{sec:model} describes the theoretical model, discusses our findings, and introduces our complementary algorithmic design. \Cref{sec:empirical_modeling} outlines the empirical approach and data. \Cref{sec:empirical_results} presents our empirical findings, and \Cref{sec:discussion_conclusion} concludes with implications for practice and future research directions.

\section{Related Work}
\label{sec:relatedwork}
This paper is related to the algorithmic fairness literature, which studies the design and evaluation of algorithms aimed to mitigate bias and improve fairness in algorithmic decision-making \parencite{dwork_fairness_2012, zemel_learning_2013, hardt_equality_2016,zafar_fairness_2017,zafar_parity_2017,geyik_fairness-aware_2019,blum_multi_2022}.
In this literature, two broad notions of fairness exist: \emph{individual fairness}, which requires that similar individuals are treated similarly by the algorithm; and \emph{group fairness}, which requires that some statistic of interest is on average equal across groups along the lines of \textit{protected attributes}.\footnote{
    Protected attributes are attributes that are protected under the law against discrimination. U.S. federal law prohibits employment discrimination based on race, gender, religion, national origin, age, disability, sexual orientation, and pregnancy.
}
Within group fairness, different definitions of fairness exist, such as demographic (or statistical) parity, equal selection, equal false-positive rates, equal false-negative rates, equal odds, equal accuracy rates, and equal positive predictive values across groups~(see \Cref{tab:screening_algos} for precise definitions and \textcite{mitchell_algorithmic_2021} for a review).
Except in trivial cases, it is impossible to simultaneously satisfy all fairness criteria  \parencite{chouldechova_fair_2017,kleinberg_inherent_2016}, so the choice of fairness criteria depends on the context and is often informed by laws, policies, and desired outcomes.

Fairness constraints are not only used to mitigate any potential bias in the algorithm but can also be used as a tool to inscribe diversity policies that proactively correct for pre-existing societal and systemic bias.
For example, in the hiring context, prior studies have shown that women are deterred from applying to male-dominated jobs because they anticipate discrimination in the hiring process \parencite{storvik_search_2008,brands_leaning_2017,bapna_rejection_2021}.
To address such pre-existing disparities, firms have adopted hiring diversity policies that increase or equalize the representation of minorities in the shortlist \parencite{shi_adoption_2018}.\footnote{
    For example, the diversity hiring policies implemented in high-tech firms such as \href{https://www.bloomberg.com/news/articles/2017-01-09/facebook-s-hiring-process-hinders-its-effort-to-create-a-diverse-workforce}{Facebook}, \href{https://newsroom.pinterest.com/en/post/our-plan-for-a-more-diverse-pinterest}{Pinterest}, \href{https://www.slideshare.net/TarynArnold/patreon-culture-deck-april-2017}{Patreon}, and \textit{LinkedIn Recruiter's} ranking algorithm \parencite{geyik_fairness-aware_2019} all seek to increase the representation of minorities in the shortlist.
}
As hiring becomes increasingly aided by algorithms, these diversity policies are implemented as algorithmic fairness constraints.
Of particular interest is the \emph{equal selection} fairness constraint \parencite{khalili_fair_2021,jiang_fair_2023}, which requires positive outcomes to be equal across groups regardless of the proportions in the baseline population.
For example, in algorithmic hiring, an equal selection constraint might require that the screening algorithm shortlists an equal number of men and women, regardless of the proportion of women in the applicant pool.\footnote{
    This is in contrast to \emph{demographic parity}, another common fairness constraint in the algorithmic hiring setting, which requires the proportion of positive outcomes across groups to be equal to the proportions in a baseline population \parencite{raghavan_mitigating_2020}.
    For example, in algorithmic screening, demographic parity may require that the proportion of women on the shortlist be equal to the proportion of women in the applicant pool.
    Whereas demographic parity ensures that bias is not introduced in the hiring process, it does not correct for pre-existing disparities.}

Although these constraints guarantee fairness on algorithmic outputs, when these outputs are used as inputs in downstream decisions, the overall effects of these constraints in either mitigating bias or increasing diversity are not guaranteed.
An emerging line of literature studies the efficacy of algorithmic fairness constraints in ``pipelines''---i.e., settings where decisions are made sequentially.
\textcite{bower_fair_2017} analyze the equal opportunity constraint in a pipeline setting and shows that individually fair algorithms, when assembled sequentially, do not necessarily guarantee fair final outcomes with respect to equal opportunity.
Similarly, \textcite{dwork_fairness_2019} analyze the individual fairness constraint and conditional parity constraints in composition settings and show that individually fair algorithms, when composed together, do not necessarily guarantee fair final outcomes.
\textcite{blum_multi_2022} propose a fair algorithm that satisfies the equality of opportunity constraint across the entire selection pipeline.
Our main contribution to this algorithmic fairness and fair pipelines literature is that we study the \emph{equal selection} constraint in a hiring pipeline setting, where decisions are made sequentially.
We propose an algorithmic design to increase the effectiveness of the equal selection constraint and demonstrate its effectiveness using empirical hiring data.

Outside the algorithmic fairness literature, our work is also related to a number of theoretical papers that study bias and fairness in hiring settings.
\textcite{kleinberg_selection_2018} provide a theoretical hiring model in the presence of implicit bias and show that the Rooney Rule can increase the proportion of minority hires while also increasing the payoff of the decision-maker (see also \textcite{celis_effect_2021}).
\textcite{fershtman_soft_2021} present a model to study the effect of ``soft'' affirmative action policies that increase the proportion of minority candidates in the candidate pool.
\textcite{lee_diversity_2021} study a 2-stage hiring setting with agents with different levels of interest in diversity and show that this difference can lower the likelihood of highly qualified candidates being hired even when they enhance diversity.
Our contribution to this theoretical hiring literature is that we explicitly model the correlation in assessment criteria between the screener and the hiring manager, which we show to be a key determinant of the effectiveness of a common diversity policy.

\section{Theoretical Framework and Implications}
\label{sec:model}

\subsection{Model Setup}

Consider a hiring context with $n_a$ applicants, each characterized by their group membership $g \in \{m, f\}$, where $m$ represents the majority group (male) and $f$ the minority group (female). The female proportion among applicants is $p_a < 0.5$. Each candidate also has an unobservable true quality $Q$, which is measurable only post-hire (e.g., via job performance).

\noindent \textbf{Two-Stage Hiring Process.} The hiring involves two sequential stages:
\begin{enumerate}
\item \textit{Algorithmic Screening}: An algorithm assigns each candidate a screening score $Q^S$ and shortlists candidates exceeding a threshold. To enhance diversity, the algorithm implements an \textit{equal selection} constraint, shortlisting an equal number of male and female candidates by setting gender-specific thresholds ($\tau^S_m, \tau^S_f$). Let $p_s$ be the proportion of women in the shortlist.
\item \textit{Human Evaluation}: The shortlisted candidates are evaluated by a hiring manager who assigns a score ($Q^H$) and hires those exceeding a common threshold ($\tau^H$), independent of gender.\footnote{
Indeed implementing a constraint on the hiring manager to hire an equal number of men and women would trivially increase the gender diversity of hires; however, such a constraint on the hiring manager would be considered a hiring quota, which is prohibited under US Employment Law (Title VII, Civil Rights Act of 1974).
This is the reason many diversity-focused hiring policies (e.g., Rooney Rule, Facebook's hiring policy \parencite{huet_facebooks_2017}, LinkedIn's screening algorithm \parencite{geyik_fairness-aware_2019}) target the initial screening decision rather than the final hiring decision.
} Let $p_h$ be the proportion of women in the hired pool.
\end{enumerate}

Table~\ref{tab:hiring_selection_rule}  summarizes these selection rules.

\begin{table}[t]
    \centering
    \caption{Stages of the hiring model}
    \label{tab:hiring_selection_rule}
    \begin{threeparttable}
                \begin{tabular}{ccc}
            \toprule
            Stage & Constraint                & Selection Rule                 \\
            \midrule
            (1)   & \makecell{Equal Selection                                  \\ \vspace{2mm} $\mathds{P}(g=f \mid y^S=1)=\mathds{P}(g=m \mid y^S=1)$}       &
            \begin{minipage}{0.4\textwidth}
                \[
                    y^S =
                    \begin{cases}
                        1 & \text{if } Q^S > \tau^S_m, g=m \\
                        1 & \text{if } Q^S > \tau^S_f, g=f \\
                        0 & \text{otherwise}
                    \end{cases}
                \]
            \end{minipage}
            \\
            \midrule
            (2)   & None                      & \begin{minipage}{0.4\textwidth}
                                                    \[
                                                        y^H =
                                                        \begin{cases}
                        1 & \text{if } Q^H > \tau^H \\
                        0 & \text{otherwise}
                    \end{cases}
                                                    \]
                                                \end{minipage} \\
            \bottomrule
        \end{tabular}
        \begin{tablenotes}
            \footnotesize
            \item \emph{Notes:} $y^S$ and $y^H$ are binary indicators of selection in the screening and hiring stages, respectively.
            Gender-specific thresholds $\tau^S_m$ and $\tau^S_f$ ensure equal selection, while $\tau^H$ is gender-neutral.
        \end{tablenotes}
    \end{threeparttable}

\end{table}

\begin{table}[t]
    \centering
    \caption{Model parameters, definitions, and assumptions}
    \label{tab:parameters}
    \renewcommand{\arraystretch}{1.5} 
    
    \begin{tabular}{m{0.15\textwidth}m{0.5\textwidth}m{0.3\textwidth}}
        \toprule
        \textbf{Parameter}                          & \textbf{Definition}                                                                                                                   & \textbf{Assumption}                                                \\
        \hline
        $\theta^S$                                  & Correlation between $Q$ and $Q^S$; measure of how good the screening algorithm is at predicting true quality                          & $\theta^S \in [0,1)$                                 \\
        \hline
        $\theta^H$                                  & Correlation between $Q$ and $Q^H$; measure of how good the hiring manager is at predicting true quality                               & $\theta^H \in [0,1)$                                               \\
        \hline
        $\theta$                                    & Correlation between $Q^S$ and $Q^H$; degree to which the screening algorithm and the hiring manager agree in their quality assessment & $\theta \in [0,1)$                                                 \\

        \hline
        $\tau^S$                                    & Quality cutoff for the screener to pass the candidate to the next round ($Q_S \geq \tau_S$) & --- \\

       \hline
        $\tau^H$                                    & Quality cutoff for the hiring manager to hire a candidate ($Q_H \geq \tau_H$) & --- \\
        
        \hline
        $\delta \coloneq \theta_m - \theta_f$       & Gender difference in correlation between the screener and the hiring manager                                                  & $\delta = 0$ (for now)    \\
        \hline
        $\delta^S \coloneq \theta^S_m - \theta^S_f$ & Predictive gender bias of screening scores                                                                                            & $\delta^S = 0$                                                     \\
        \hline
        $\delta^H \coloneq \theta^H_m - \theta^H_f$ & Predictive gender bias of the hiring manager scores                                                                                   & $\delta^H = 0$                                                     \\
        \hline
        $\alpha$                                    & Mean quality difference between men and women; positive $\alpha$ implies women have higher mean quality than men                      & $\alpha = 0$. We extend the model in \Cref{apx:model_diff_in_qual} \\
        \hline
        $\beta^S$                                   & Systematic gender bias of screening scores                                                                                            & $\beta^S = 0$                                                      \\
        \hline
        $\beta^H$                                   & Systematic gender bias of hiring manager scores                                                                                       & $\beta^H = 0$                                                      \\
        \bottomrule
    \end{tabular}

\end{table}

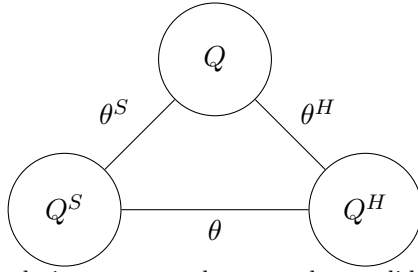
\begin{figure}[t]
    \centering
    \caption{The correlation structure between $Q, Q^S, Q^H$}
    \label{fig:Q}

    \begin{tikzpicture}
        \node[circle,draw, minimum size = 1.5cm] (A) at (0,0) {$Q^S$};
        \node[circle,draw, minimum size = 1.5cm] (B) at (2,2) {$Q$};
        \node[circle,draw, minimum size = 1.5cm] (C) at (4,0) {$Q^H$};

        \draw (A) -- node[above left]{$\theta^S$} (B);
        \draw (B) -- node[above right]{$\theta^H$} (C);
        \draw (C) -- node[below]{$\theta$} (A);
    \end{tikzpicture}

    \begin{minipage}{1\textwidth}
            \begin{footnotesize}
                \begin{singlespace}
                    \emph{Notes:}  
                    This figure illustrated the correlation structure between the candidates’ true quality ($Q$), the algorithm’s
quality estimate ($Q^S$ ), and the hiring manager’s quality estimate ($Q^H$). The $\theta$ values represent the correlations between these scores.
                \end{singlespace}
            \end{footnotesize}
        \end{minipage}
\end{figure}

\noindent \textbf{Quality Scores Model.} We model the scores $(Q, Q^S, Q^H)$ as following a multivariate Gaussian distribution, potentially with distinct distributions for male and female candidates:
\begin{align}
(Q_m, Q_m^S, Q_m^H) &\sim \mathcal{N} \left( \begin{bmatrix} 0 & 0 & 0 \end{bmatrix}, \begin{bmatrix}
1          & \theta^S & \theta^H \\
\theta^S & 1          & \theta   \\
\theta^H & \theta   & 1
\end{bmatrix} \right) \\
(Q_f, Q_f^S, Q_f^H) &\sim \mathcal{N} \left( \begin{bmatrix} \alpha & \alpha + \beta^S & \alpha + \beta^H \end{bmatrix}, \begin{bmatrix}
1          & \theta^S - \delta^S & \theta^H - \delta^H \\
\theta^S - \delta^S & 1          & \theta - \delta   \\
\theta^H - \delta^H & \theta - \delta   & 1
\end{bmatrix} \right)
\end{align}
Without loss of generality, male candidates have a mean vector of zero, and female candidates may differ by parameters $(\alpha, \beta^S, \beta^H)$.
The correlation structure is defined by a positive semi-definite covariance matrix with parameters $(\theta, \theta^S, \theta^H)$.
The correlation structure may differ by gender, where the difference is parameterized by $(\delta, \delta^S, \delta^H)$.
For now, we assume that $(\delta, \delta^S, \delta^H)=0$.
\Cref{tab:parameters} summarizes all model parameters and assumptions, and Figure~\ref{fig:Q} provides a visual representation of the quality score model.

\subsection{Theoretical results}
\label{sec:theoretical_results}

We analyze how the gender diversity of hires, $p_h$, and the expected quality of hires, $E[Q_{h}]$, vary as functions of the firm's design parameters with respect to the screening algorithm---$\theta, \delta, \theta^S$.\footnote{ 
Not all model parameters are design parameters that can be controlled by the firm.
    For a given candidate, $Q$ is fixed, and estimation of $Q^H$ is delegated to the hiring manager, which fixes $\theta^H$.
        The firm has control over the screening algorithm, and thus how $Q^S$ is estimated.
    Therefore, the design parameters that the firm can control are $\theta^S$ (i.e., how good the screening algorithm is in predicting true quality), and $\theta$ (how similar the screening algorithm is compared to the hiring manager in assessing quality).}
\subsubsection{Effects on hire diversity}

\newcommand\cmdphVsTheta{The effectiveness of the equal selection constraint ($p_h$) decreases as the correlation ($\theta$) between algorithmic scores and hiring manager scores increases.}
\begin{mdframed}
    \begin{proposition}
        \label{prop:ph_vs_theta}
        \cmdphVsTheta{}
    \end{proposition}
\end{mdframed}

\begin{corollary}
When screening and hiring manager scores are perfectly uncorrelated $(\theta = 0)$, the equal selection constraint effectively balances the gender proportion of hires.
In contrast, when the scores are perfectly correlated $(\theta = 1)$, equal selection has no effect on the gender proportion of hires. Under partial correlation $(0<\theta<1)$, higher values of $\theta$ lead to decreasing effectiveness of the constraint.

\end{corollary}
\begin{figure}[H] 
    \caption[]{Female proportion of hires ($p_h$) vs. correlation parameter ($\theta$)}
    \centering
    \includegraphics[width=0.5\textwidth, keepaspectratio]{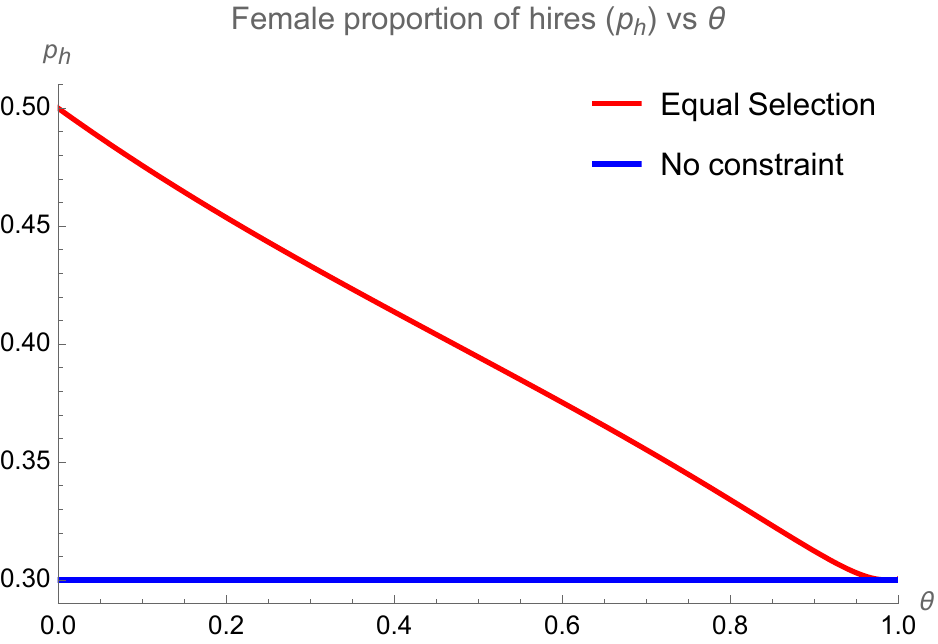}
    \label{fig:pr_fhire_vs_theta}
    \begin{minipage}{1\textwidth}
        \begin{footnotesize}
            \begin{singlespace}
                \emph{Notes:}
                This figure plots the female proportion of hires, $p_h$, as a function of the correlation parameter, $\theta$. The proportion of women in the applicant pool is fixed at $p_a=0.3$.
            \end{singlespace}
        \end{footnotesize}
    \end{minipage}
\end{figure}

A formal proof is provided in \Cref{sec:proof_theta_vs_ph}.
Here, we provide an intuitive explanation.
Under equal selection, the female shortlist threshold ($\tau^S_f$) is adjusted such that an equal number of women and men are shortlisted.
Since there are more men than women in the applicant pool, the shortlist threshold for women will be lower compared to men ($\tau^S_f < \tau^S_m$).
This means that the average $Q^S$ score of women will be lower compared to men in the shortlist.
When the screening and hiring manager scores are perfectly correlated ($\theta = 1$), $Q^S=Q^H$, this translates to lower average $Q^H$ score for women.
So, even though there are an equal number of male and female candidates in the shortlist, the shortlisted female candidates will be less likely to get hired compared to the male candidates.
On the other extreme, when the two scores are perfectly uncorrelated ($\theta = 0$), the $Q^S$ scores are independent of $Q^H$.
Even though shortlisted female candidates have lower average $Q^S$ score than male candidates, they have the same average $Q^H$ score.
Therefore, when $\theta=0$, the probability that a female is hired equals $\frac{1}{2}$.
In partially correlated cases, the outcomes lie between these two extremes: the gender diversity outcomes will be worse when algorithmic and human evaluations align closely.

\newcommand\cmdphVsDelta{The female proportion of hires ($p_h$) decreases with the gender difference in the correlation parameter ($\delta$).}
    \begin{mdframed}
        \begin{proposition}
            \label{prop:ph_vs_delta}
            \cmdphVsDelta{}
        \end{proposition}
    \end{mdframed}
    
    \begin{figure}[H] 
        \caption[]{Female proportion of hires ($p_h$) vs. gender difference in correlation parameter ($\delta$)}
        \centering
        \includegraphics[width=0.5\textwidth, keepaspectratio]{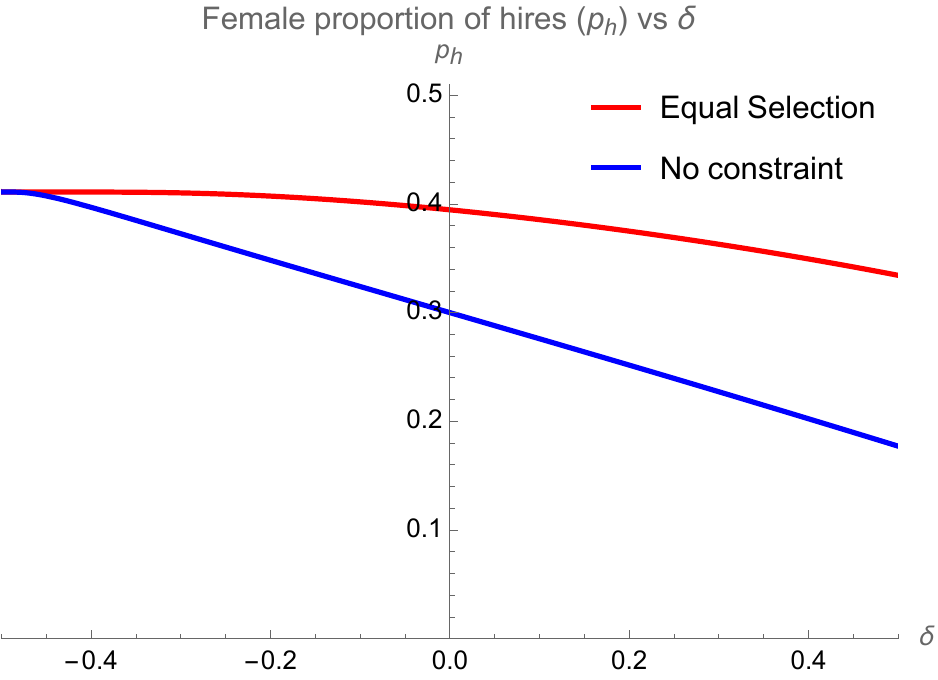}
        \label{fig:pr_fhire_vs_delta}
        \begin{minipage}{1\textwidth}
            \begin{footnotesize}
                \begin{singlespace}
                    \emph{Notes:}
                    This figure plots the female proportion of hires, $p_h$, as a function of the gender difference in correlation parameter, $\delta$. The proportion of women in the applicant pool is $p_a=0.3$.
                \end{singlespace}
            \end{footnotesize}
        \end{minipage}
    \end{figure}

    We provide the proof in \Cref{sec:proof_delta_vs_ph}. The intuition is as follows:
    $\delta > 0$ means that the screening algorithm is less predictive of the hiring manager's evaluation for female candidates, which means lower $Q^H$ scores for female candidates compared to men in the shortlist.
    This in turn means that the probability of a female being hired decreases.
    Therefore, any gender-specific discrepancies in evaluation consistency reduces female proportion of hires, both with and without the equal selection constraint.
    
    \subsubsection{Effects on hire quality}
    
\newcommand\cmdEQvsTheta{Conditional on the predictive accuracy of the screening algorithm ($\theta^S$) and the hiring manager ($\theta^H$), the average hire quality decreases as the correlation ($\theta$) between algorithmic scores and hiring manager scores increases in the space $\theta \in [0, \min\{\frac{\theta^S}{\theta^H}, \frac{\theta^H}{\theta^S}\}]$, with hire quality reaching a global maximum at $\theta=0$.}
\begin{mdframed}
    \begin{proposition}
        \label{prop:eq_vs_theta}
        \cmdEQvsTheta{}
    \end{proposition}
\end{mdframed}

\begin{figure}[H] 
    \caption[]{Expected quality of hire ($E[Q_h]$) vs. $\theta$, $\theta^S$}
    \centering
    \includegraphics[width=\textwidth, keepaspectratio]{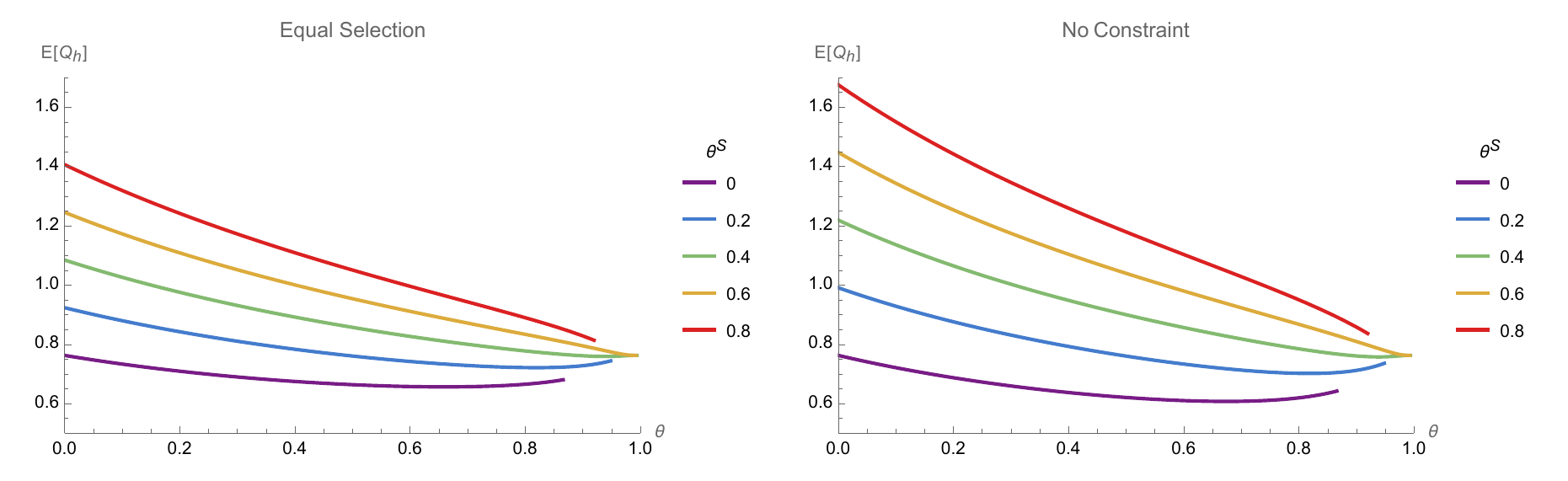}
    \label{fig:eq_hire_vs_theta_thetas}
    \vspace*{-0.2in}
    \begin{minipage}{1\textwidth}
        \begin{footnotesize}
            \begin{singlespace}
                \emph{Notes:}
                This figure plots the expected quality of hires, $E[Q_h]$, as a function of the correlation parameter, $\theta$ for different $\theta^S$ values. The rest of the parameters are fixed at $\theta^H=0.5, p_a=0.3, \delta=0$.
            \end{singlespace}
        \end{footnotesize}
    \end{minipage}
\end{figure}

We provide the formal proof in \Cref{sec:proof_theta_vs_eq} and provide an intuitive explanation here.
The screening score ($Q^S$) and hiring manager score ($Q^H$) act as two noisy signals providing information about the candidate's true quality ($Q$).
The individual informativeness of these signals (i.e., $\theta^S, \theta^H$) is fixed.
A key principle, discussed by \textcite{clemen_limits_1985}, states that when combining signals, \textit{less correlated} signals collectively provide more information about the underlying value than \textit{highly correlated} signals (for a fixed level of individual signal informativeness).
Thus, higher correlation leads to redundant information, reducing the overall quality gains from screening.

\subsection{Implications for Algorithm Design}
    \label{sec:theoretical_discussion}
    Our theoretical analysis offers clear guidance on designing screening algorithms to simultaneously maximize hire quality and workforce diversity. 
    Specifically, our findings highlight three key points: (1)~hire quality increases with the predictive accuracy of the screening algorithm ($\theta^S$), (2)~the proportion of female hires decreases with increased correlation ($\theta$) between algorithmic scores and hiring manager evaluations under equal selection constraints (Proposition~\ref{prop:ph_vs_theta}), and (3)~conditional on predictive accuracy ($\theta^S$), higher correlation ($\theta$) reduces expected hire quality (Proposition~\ref{prop:eq_vs_theta}). 
    Therefore, the optimal strategy involves selecting screening algorithms with high $\theta^S$ (good at predicting true quality) but low $\theta$ (distinct from human evaluations), making the algorithms \textit{complementary} to human assessments.

    \subsubsection{Algorithmic Selection with Independent Parameters}

    Consider a scenario where a firm chooses between two screening algorithm vendors. 
    Both algorithms are equally accurate at predicting true quality (i.e., $\theta^S_1 = \theta^S_2$) but differ in their correlation with hiring managers' evaluations (i.e., $\theta_1 \neq \theta_2$). Such differences can arise if algorithms rely on varying feature sets.
    
    The firm should choose the algorithm with lower correlation ($\theta$) to hiring manager assessments. Despite similar predictive performance, lower correlation algorithms offer less redundant information, thus enhancing both diversity and expected hire quality under equal selection constraints.

    \subsubsection{Balancing Predictive Accuracy and Managerial Complementarity}
    
    In practice, a firm often designs a screening algorithm with a fixed information source, such as resumes, creating inherent trade-offs between predictive accuracy ($\theta^S$) and complementarity to human evaluations ($\theta$). 
    We outline several algorithm training strategies based on available target variables (Figure~\ref{fig:training-options} illustrates these visually):

    \begin{figure}[t]
        \centering
        \begin{tikzpicture}
            \node[circle,draw, thick, minimum size = 1.5cm] (A) at (0,0) {$Q^S$};
            \node[circle,draw, minimum size = 1.5cm] (B) at (0,-2) {$\mathcal{A}$};
            \draw[-{Latex[length=2mm]}] (A) -- (B);
            \node[below] at (0,-3) {(1)};

            \node[circle,draw, thick, minimum size = 1.5cm] (C) at (2,0) {$Q^H$};
            \node[circle,draw, minimum size = 1.5cm] (D) at (2,-2) {$\mathcal{A}$};
            \draw[-{Latex[length=2mm]}] (C) -- (D);
            \node[below] at (2,-3) {(2)};

            \node[circle,draw, thick, minimum size = 1.5cm] (E) at (4,0) {$Q$};
            \node[circle,draw, minimum size = 1.5cm] (F) at (4,-2) {$\mathcal{A}$};
            \draw[-{Latex[length=2mm]}] (E) -- (F);
            \node[below] at (4,-3) {(3)};

            \node[circle,draw, thick, minimum size = 1.5cm] (G) at (6,0) {$Q$};
            \node[circle,draw, thick, minimum size = 1.5cm] (H) at (8,0) {$Q^H$};
            \node[circle,draw, minimum size = 1.2cm] (I) at (7,-2) {$\mathcal{A}$};
            \draw[-{Latex[length=2mm]}] (G) -- (I);
            \draw[-{Latex[length=2mm]}] (H) -- (I);
            \node[below] at (7,-3) {(4)};


        \end{tikzpicture}

        \caption{Target variable options for training a screening algorithm $\mathcal{A}$.}
        \label{fig:training-options}

        \label{fig:my_label}
    \end{figure}
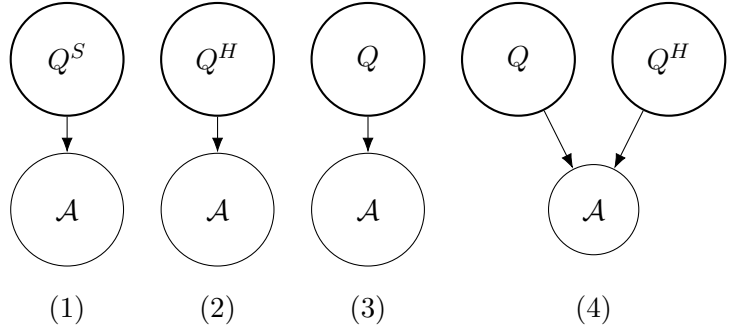

\begin{itemize}

\item \textbf{Option 1 (Historical Human Screener Scores):} Training on historical screening scores ($Q^S$) provides abundant data but has limited control over both $\theta^S$ and $\theta$ since past screening scores are proxies, not perfect predictors of true quality.

\item \textbf{Option 2 (Hiring Manager Scores):} Training directly on hiring manager evaluations ($Q^H$) maximizes correlation ($\theta$), diminishing the diversity benefits of equal selection.

\item \textbf{Option 3 (True Quality Scores):} Training on actual job performance data ($Q$) maximizes predictive accuracy ($\theta^S$) but offers no direct control over managerial correlation ($\theta$).

\item \textbf{Option 4 (Multi-objective Learning):} Training simultaneously on true quality ($Q$) and hiring manager evaluations ($Q^H$) using techniques such as adversarial learning. This balances high predictive accuracy and managerial complementarity, optimizing both diversity and hire quality simultaneously.

\end{itemize}

While these practical strategies describe how different target variables affect predictive accuracy and managerial correlation, it is also crucial to consider theoretical limitations in simultaneously optimizing these parameters when the overall information available is fixed.

We now specifically examine this scenario, where the total information regarding true candidate quality ($Q$), conditional on the algorithm's scores ($Q^S$) and the hiring manager's evaluations ($Q^H$), remains constant. Formally, the information given by $Q^S$ and $Q^H$ about $Q$ is given by the conditional entropy $H(Q|Q^S, Q^H)$ which is:
\begin{equation}
H(Q|Q^S, Q^H) = \frac{1}{2}\cdot \log\left(\frac{2e\pi\cdot \text{Det}\left( \begin{bsmallmatrix}
1 & \theta^S & \theta^H \\
\theta^S & 1 & \theta \\
\theta^H & \theta & 1
\end{bsmallmatrix}\right)}{1-\theta^2}\right)
\end{equation}

\noindent Under these fixed-information conditions, any attempt to maximize predictive accuracy ($\theta^S$) will inherently constrain efforts to reduce correlation ($\theta$) and vice versa. 
If we express $\theta^S$ as a function of $\theta$, with fixed information $H_0$, we get:
\begin{equation}
\theta^{S} =
\theta \cdot \theta^{H} \pm
\sqrt{\left(1-(\theta^{H})^{2}\right) \cdot \left(1-\theta^{2}\right)
-\frac{(1-\theta)\,e^{2H_0}}{2e\pi}}
\end{equation}
\noindent Figure~\ref{fig:equalinfo} demonstrates these trade-offs, showing pairs of ($\theta$, $\theta^S$) that yield constant information. Maximizing $\theta^S$ initially enhances predictive performance, but further reductions in correlation ($\theta$) inevitably decrease $\theta^S$. 
But, notably, even after reaching peak $\theta^S$, further reducing $\theta$ can still increase expected hire quality due to the reduction in redundant information in a single stage.\footnote{
  Decreasing $\theta$ affects the expected quality of hires via two channels: (1)~there is a direct effect of $\theta$, where decreasing $\theta$ increases $E[Q_h]$, and (2)~there is an indirect effect via $\theta^S$, where decreasing $\theta$ also decreases $\theta^S$, which in turn decreases $E[Q_h]$.
   Interestingly, the net effect of decreasing $\theta$ still increases $E[Q_h]$ even though $\theta^S$ is simultaneously decreasing---meaning that the direct increase in $E[Q_h]$ due to the decrease in $\theta$ offsets the indirect decrease in $E[Q_h]$ due to decreasing $\theta^S$.
}
    
    \begin{figure}[t] 
        \caption[]{Simultaneously optimizing $\theta^S$ and $\theta$ with fixed information about $Q$}
        \centering
        \includegraphics[width=\textwidth, keepaspectratio]{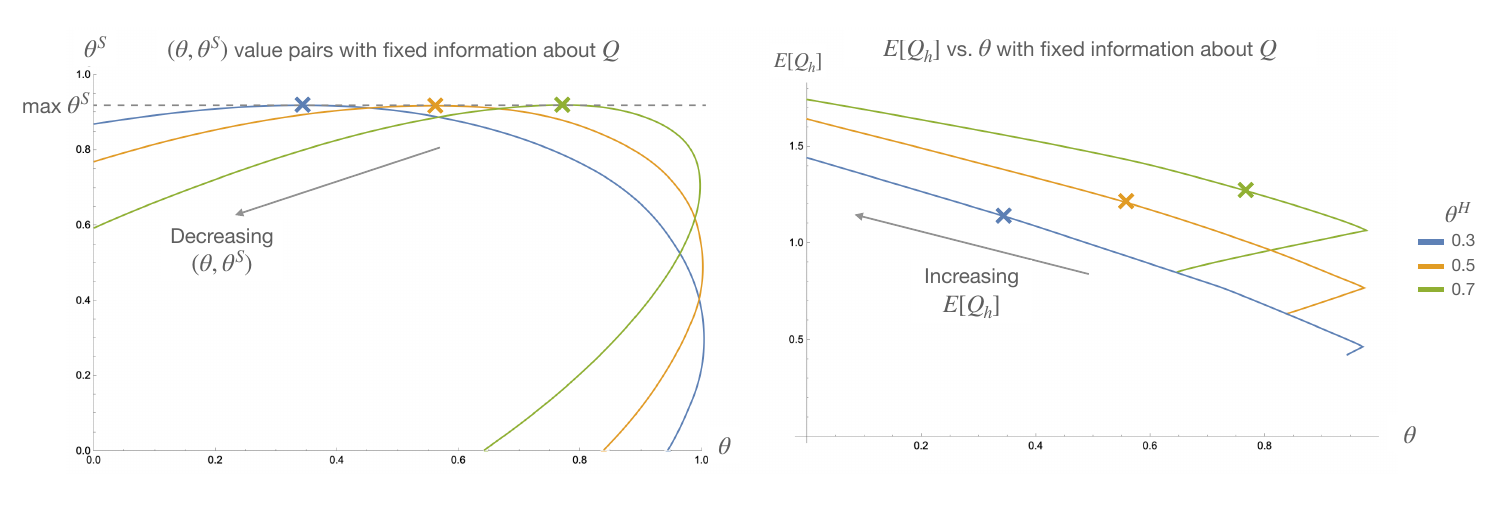}
        \label{fig:equalinfo}
        \begin{minipage}{0.9\textwidth}
            \begin{footnotesize}
                \begin{singlespace}
                    \emph{Notes:}  
                    The left panel plots equal-information $(\theta, \theta^S)$ pairs yielding the same conditional entropy $H(Q|Q^S, Q^H)=0.5$.
                    The 'x' marks the point where $\theta^S$ is maximized.
                    The right panel plots expected hire quality $E[Q_h]$ as a function of $\theta$ using these equal-information pairs.
                    $\theta^H$ is held fixed at $\theta^H \in {0.3, 0.5, 0.7}$.
                \end{singlespace}
            \end{footnotesize}
        \end{minipage}
    \end{figure}
    
\subsubsection{Directly Minimizing Gender Differences in \texorpdfstring{$Q^H$}{Q\^H} Scores}

Another approach is to directly minimize gender differences in $Q^H$ scores in the shortlist. This can be implemented by training two predictors—one for $Q$ and another for $Q^H$—and then shortlisting candidates with high predicted $Q$ scores while minimizing gender differences in predicted $Q^H$ scores.

This approach deliberately selects male and female candidates with similar hiring manager scores, maximizing the likelihood of gender-balanced hiring. While not minimizing $\theta$ to zero (and thus not maximizing hire quality to the fullest extent), this method effectively balances diversity and quality goals while being simpler to implement than adversarial approaches.

In summary, the firm's goal in algorithm design should be to create a screening system that complements, rather than replicates, managerial evaluations, thereby maximizing both quality and diversity outcomes under equal selection constraints.

\section{Data and Empirical Methodology}
    \label{sec:empirical_modeling}

    Our theoretical analysis demonstrates that the effectiveness of the equal selection constraint depends significantly on parameters such as the correlation between screening algorithms and hiring manager evaluations ($\theta$), and gender differences in these evaluations ($\delta$).
    The optimal screening algorithm is one that is trained on both true quality and the hiring manager's assessment of quality.
    However, in practice, screening algorithms are typically trained on historical human screening decisions rather than true job performance, primarily due to data availability.\footnote{For example, \textit{LinkedIn Recruiter's} recommendation algorithm is trained on the human recruiter's decisions since it has no visibility into the true job performance of the candidates.} 
    This raises an empirical question: if firms train screening algorithms based on historical recruiter decisions, how effective will equal selection constraints be in improving diversity outcomes?
    
    In this section, we describe our empirical approach to addressing this question. 
    We estimate model parameters using actual hiring data from multiple firms, and then use these parameters to estimate the effectiveness of equal selection constraints across different job contexts. 
    We also benchmark these outcomes against our proposed complementary screening algorithm and other commonly used fairness metrics using simulation.

    \subsection{Data Description}
    We use Applicant Tracking System (ATS) data from eight U.S.-based technology companies, provided by an HR analytics software vendor. 
    This dataset includes detailed records for 799,000 external job applicants (60\% male, 40\% female) across 3,608 unique job postings. Each record captures candidate attributes (such as gender and experience), resumes, job posting details, and outcomes at each hiring stage (screening, first interview, subsequent interviews, and offers). 
    
    Table~\ref{tab:counts} summarizes the number of applicants and job postings by job category. 
    Although specific hiring processes can vary slightly across firms, the typical hiring sequence for external applicants, as shown in Figure~\ref{fig:hiring_funnel}, involves four main stages: Screening, First Interview, Subsequent Interviews, and Offer. 
    Initially, applicants undergo a screening stage. 
    Those who pass the screening proceed to the first interview stage, followed by subsequent interviews, and finally receive an offer if selected.\footnote{
    For our empirical analysis, we specifically focus on two critical stages: the initial screening stage—where the equal selection constraint is applied—and the subsequent first interview stage. 
    Although the actual hiring process is multi-staged, this simplified focus remains appropriate. To illustrate, consider a scenario with an applicant pool comprising 70\% males and 30\% females. 
    With an equal selection constraint, the shortlisted candidates following the screening stage would consist of an equal gender split (50/50). 
    However, the hiring manager in the first interview stage might partially reverse this constraint, resulting in a gender ratio such as 60/40. 
    Provided that selections in subsequent stages are unbiased---a central assumption in our theoretical model---this revised gender ratio would persist throughout the remainder of the hiring process.
    }    
    On average, a typical job posting attracts 233 applicants, with about 36 advancing past the initial screening, approximately 7 candidates progressing beyond the first interview, and around 2 receiving offers.
    
    \begin{table}[t] 
        \centering
        \caption{Number of applicants and job postings by job category}
        \begin{adjustbox}{max width=\textwidth}
            \begin{threeparttable}
                \begin{tabular}{lrr}
    \toprule
    Job Category                        & N Applicants  & N Jobs \\
    \midrule
    Engineering \& Technical            & 214,943 & 1,178  \\
    Product \& Design                   & 130,669 & 534    \\
    Sales \& Marketing                  & 92,559  & 391    \\
    Legal \& PR                         & 75,955  & 332    \\
    Other                               & 70,864  & 53     \\
    Finance \& Accounting               & 69,536  & 308    \\
    Biz Dev \& Operations               & 51,523  & 299    \\
    Human Resources                     & 48,122  & 246    \\
    Customer Service \& Acct Management & 42,199  & 238    \\
    \midrule
    Overall                             & 799,108 & 3,608  \\
    \bottomrule
\end{tabular}

            \end{threeparttable}
        \end{adjustbox}
        \label{tab:counts}
    \end{table}

    \begin{figure}[t] 
        \caption{Hiring funnel}
        \centering
        \includegraphics[width=5in, keepaspectratio]{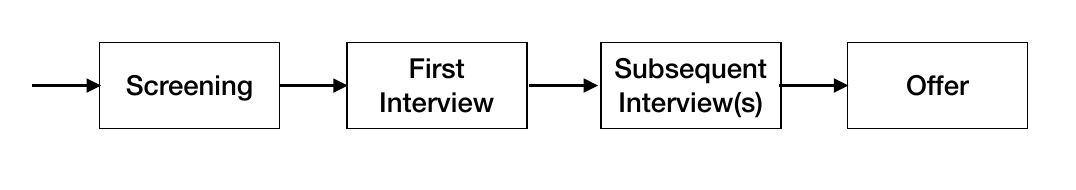}
        \label{fig:hiring_funnel}
    \end{figure}

    \subsection{Empirical Strategy}
    \label{sec:param_estimation}
    
    Our empirical approach consists of three key steps:
    
    \paragraph{Step 1: Estimating Screening and Hiring Manager Scores.}
    The Applicant Tracking System (ATS) provides only binary outcomes (screening and interview decisions). To analyze the effectiveness of equal selection constraints, we first derive continuous quality scores for both the screening and hiring manager stages.
    We achieve this by training two machine learning (ML) models separately: one predicting screening decisions and the other predicting hiring manager decisions, using candidate resume texts and job descriptions as inputs. Further technical details about these models are discussed in \Cref{sec:ml_models}.\footnote{Note that we observe binary screening decisions and not the assessed scores by the screener and we observe binary hiring manager decisions only for shortlisted candidates. By building the ML models, we can infer continuous quality scores for all candidates and all stages.}
    
    We employ BigBird, a transformer model specifically optimized for processing long text documents~\parencite{zaheer_big_2020}.
    To address selection bias arising from observing hiring manager decisions only for shortlisted candidates, we apply inverse propensity weighting to re-weight observations based on their probability of passing the screening stage. 
    This provides unbiased estimates of hiring manager evaluations for all candidates.
    
    Predicted decision probabilities from the ML models are converted into quality scores using a Gaussian copula transformation, aligning with our theoretical assumption of multivariate Gaussian distributions for quality scores.\footnote{
        Gaussian copulas are multivariate Gaussian distributions, whose marginals are uniformly distributed.
        They offer a flexible way to disentangle multivariate Gaussian distribution as a product of uniform marginal distributions and a Gaussian copula that ``couples'' them (See~\textcite{joe_dependence_2014,nelsen_introduction_2007} for a reference on copulas).
        Formally, the joint empirical distribution of quality scores $(\hat{q}, \hat{q}^S, \hat{q}^H)$ has CDF $F_{\hat{q},\hat{q}^S,\hat{q}^H}(x,y,z;\Sigma) = C(F_{\hat{q}}(x), F_{\hat{q}^S}(y), F_{\hat{q}^H}(z))$.
        Here, $C$ is the 3-dimensional Gaussian copula, $C(u,v,k) = \Phi(\Phi^{-1}(u), \Phi^{-1}(v), \Phi^{-1}(k))$, and  $\Phi$ is the CDF of a multivariate Gaussian distribution.
        This transformation ensures that we stay close to the theoretical model, which assumes that the quality scores have a multivariate normal distribution. 
    }
    We show in \Cref{apx:additional_analyses} that the Gaussian copula has good goodness-of-fit measures on our empirical data compared to other copulas.
    The predicted probabilities are then transformed into quality scores via quantiles:
        \begin{align}
        \hat{q}^S_{i,j} &= \textit{Quantile}(\hat{p}^S_{i,j}, \bm{\hat{p}}^S_j), \\
        \hat{q}^H_{i,j} &= \textit{Quantile}(\hat{p}^H_{i,j}, \bm{\hat{p}}^H_j).
        \end{align}
    
    \paragraph{Step 2: Parameter Estimation ($\theta$ and $\delta$).}
    Using the recovered continuous quality scores, we estimate the critical parameters:
    
    \begin{itemize}
    \item \textbf{Correlation Parameter ($\theta$)}: We estimate the correlation between screening scores and hiring manager evaluations for each job posting using the Spearman rank correlation coefficient. This measure is robust to transformations and widely used in practice.\footnote{Using pearson correlation leads to highly similar results.}
    \begin{equation}
    \hat{\theta}j = \text{Spearman}(\bm{\hat{q}}{j}^S, \bm{\hat{q}}_{j}^H).
    \end{equation}
    
    \item \textbf{Gender Difference Parameter ($\delta$)}: We calculate the difference in correlation between male and female candidates' evaluations for each job posting.
    \end{itemize}
    
    These parameters are aggregated across job postings, weighted by the number of applicants per job, to derive representative average values.
    The rest of the parameters ($p_a, \tau^S, \tau^H$), are observed directly from the data.\footnote{
        We set the job-specific shortlist and hiring manager thresholds based on the actual size of the shortlist and finalist observed in the data.
       In doing so, we conceive the thresholds as exogenous variables.
       For example, the firm may have a limited budget to interview candidates and can only afford to interview a certain number of candidates.
       The shortlist and hiring manager threshold is therefore set based on the observed size of the shortlist and finalist respectively.
    }
    
    \paragraph{Step 3: Counterfactual Policy Simulation.}
    Using the estimated parameters, we conduct counterfactual simulations to evaluate the effectiveness of equal selection constraints in enhancing workforce diversity.
    We compare the performance of these constraints against our proposed complementary screening algorithm and other widely adopted fairness criteria, allowing us to empirically demonstrate their relative effectiveness.
    
    \subsection{ML Model Details and Performance}
    \label{sec:ml_models}
    
    To train the screening and hiring manager models, we consolidate each candidate's resume with relevant job information—such as company name, job title, business unit, employment type, location, skills, and keywords—into a unified input document. 
    The skills and keywords are sourced from a comprehensive skills dictionary developed through an extensive analysis of LinkedIn profile data.
    
    We partition the dataset into training (80\%), validation (10\%), and hold-out test sets (10\%), stratifying by job postings to ensure representativeness and robust model evaluation. 
    We follow \textcite{sun_how_2019} for picking the optimal hyperparameters and select them based on validation performance (area under the ROC curve): \texttt{Epochs=3, Batch Size=14, Learning Rate=2e-5, Weight Decay=2e-5}.
    
    The screening model training/evaluation set consists of 725,351 observations, with a hold-out test set of 73,757 observations. 
    The hiring manager model training/evaluation set includes 106,419 observations, with a hold-out test set comprising 11,357 observations.
    
    Model evaluation indicates strong predictive performance: the screening model achieves an AUC score of 0.83, while the hiring manager model achieves an AUC of 0.68. 
    We see no significant differences in predictive performance between male and female candidates. 
    Additional model performance details can be found in \Cref{apx:clf_models}.
    
    \subsection{Inverse Propensity Weighting}
    
    Since hiring manager decisions are only available for candidates who have passed the screening stage, we employ inverse propensity weighting to mitigate selection bias. 
    Specifically, candidates less likely to be shortlisted (based on screening predictions) receive higher weights, while those more likely receive lower weights.
    This adjustment ensures unbiased estimation of hiring manager scores across the full applicant pool as long as there is noise in the selection process (see \textcite{cowgill_bias_2020}).
    
    \section{Empirical Results}
    \label{sec:empirical_results}
    
    This section presents empirical findings based on our analysis of the hold-out test set.

    \subsection{Parameter estimates}
    \label{sec:param_estimates}

    \begin{figure}[t] 
        \caption{Distribution of parameter estimates across job postings}
        \centering
        \includegraphics[width=4in, keepaspectratio]{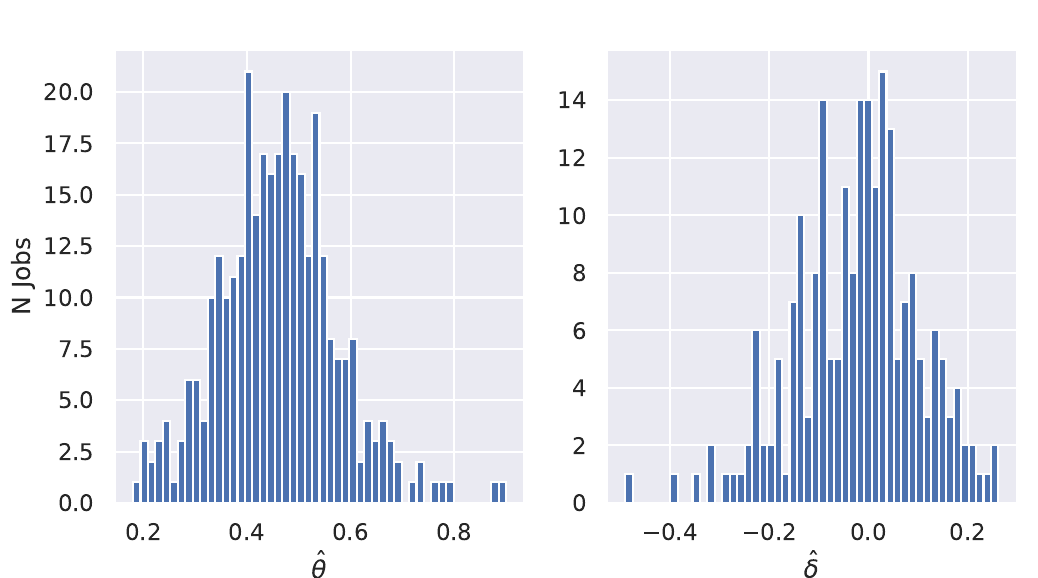}
        \label{fig:param_dist}
    \end{figure}

    \begin{table}[t] 
        \centering
        \caption{Average parameter estimates}
        \label{tab:theta_estimates}
        \begin{adjustbox}{max width=\textwidth}
            \begin{threeparttable}
                \begin{tabular}{lrr}
    \toprule
    Job Category                        & $\hat{\theta}$ & $\hat{\delta}$ \\
    \midrule
    Finance \& Accounting               & 0.493          & -0.019         \\
    Engineering \& Technical            & 0.449          & -0.009         \\
    Sales \& Marketing                  & 0.442          & 0.003          \\
    Product \& Design                   & 0.441          & 0.031          \\
    Customer Service \& Acct Management & 0.436          & -0.02          \\
    Biz Dev \& Operations               & 0.413          & -0.014         \\
    HR                                  & 0.403          & -0.074         \\
    Legal \& PR                         & 0.348          & -0.016         \\
    Other                               & 0.245          & -0.332         \\
    \midrule
    Average                             & 0.434          & -0.007         \\
    \bottomrule
\end{tabular}

                \begin{tablenotes}[flushleft]
                    \footnotesize
                    \item \emph{Notes: This table reports the average parameter estimates for each job category. We estimate the parameters at the job posting level and aggregate it up to the job category level for all jobs in the hold-out test set.
                    }
                \end{tablenotes}
            \end{threeparttable}
        \end{adjustbox}
    \end{table}

    We estimate model parameters ($\theta$, $\delta$) separately for each job posting, using the methodology in \Cref{sec:param_estimation}, and plot their distribution in \Cref{fig:param_dist}. 
    \Cref{tab:theta_estimates} reports average parameter values aggregated by job category.

    The average correlation parameter estimate, $\hat{\theta}$, is 0.43, but varies considerably across jobs.
    Higher correlation is typically observed in technical roles requiring ``hard skills'' (e.g., Finance \& Accounting, Engineering \& Technical). 
    In contrast, lower correlation occurs in roles emphasizing ``soft skills'' (e.g., HR, Legal \& PR).
    A likely explanation is that hard skills are more readily assessable from resumes, thus aligning screener and hiring manager evaluations more closely.

    Regarding gender differences, the overall average estimate for $\delta$ is -0.007, suggesting that, on average, the screening criteria align similarly for men and women.\footnote{
        Note that the ``Other'' category has a high estimate of $\delta$, but this is likely due to the small sample size (only 214 applicants in this category in the hold-out test set).}
    But, there is substantial variability across job postings, with estimates ranging from -0.4 to 0.2.

    \subsection{Effectiveness of the Equal Selection Constraint}

    We estimate the impact of equal selection constraints using counterfactual simulations. 
    For each job posting with underrepresented female applicants ($p_a < 0.5$), we simulate outcomes using the estimated job-specific parameters.

    \Cref{tab:parametric_simulation_job_cat} reports the aggregate results by job category. 
    The equal selection constraint raises the proportion of women from an average of 31\% in the applicant pool to 50\% in the shortlist (by design). 
    However, this proportion decreases to 41\% among finalists, showing a modest overall improvement.

    The effectiveness varies notably across job categories. 
    For example, Engineering \& Technical roles achieve only a 36\% representation of women in the finalist stage, indicating limitations in equal selection effectiveness in highly technical fields.
    
    \subsection{Test of propositions}
    \label{ssec:test_of_propositions}
    We next exploit the variation in parameter estimates across jobs to empirically test \Cref{prop:ph_vs_theta} and \Cref{prop:ph_vs_delta} using the following regression specification:
    
    \begin{equation}
        \label{eq:prop_1_reg}
        p_{h,j} = \beta_0 + \beta_1 p_{a,j} + \beta_2 \theta_j + \beta_3 \theta^2_j + \beta_4 \delta_j + \beta_5 \delta^2_j + \epsilon_j
    \end{equation}
    
    where $p_{h,j}$ is the estimated proportion of women in the finalist pool for job $j$, $p_{a,j}$ is the observed proportion of women in the applicant pool for job $j$, $\theta_j$ and $\delta_j$ are the estimated correlation and gender difference parameters for job $j$ respectively.
    We include the squared terms of $\theta_j$ and $\delta_j$ to capture non-linear relationships.
    
    We estimate the model using the hold-out test set with and without the equal selection constraint and report the results in \Cref{tab:prop_1_reg}.
    
    Consistent with the theoretical predictions, we find a negative relationship between $\theta$ and $p_h$ under the equal selection constraint and a negative relationship between $\delta$ and $p_h$ with and without the equal selection constraint.
    
    \begin{table}[t] 
        \centering
        \caption{Empirical test of propositions}
        \label{tab:prop_1_reg}
        \begin{adjustbox}{max width=\textwidth}
                \begingroup
\centering
\begin{tabular}{lcc}
   \tabularnewline \midrule \midrule
                    & No Constraint   & Equal Selection \\   
   Model:           & (1)             & (2)\\  
   \midrule
   \emph{Variables}\\
   $p_a$            & 1.359$^{***}$   & 1.023$^{***}$\\   
   & (0.1402)        & (0.1400)\\   
   $\theta$         & -0.2409         & -0.7005$^{**}$\\   
   & (0.2757)        & (0.2754)\\   
   $\theta^2$  & 0.2035          & 0.6057$^{**}$\\   
   & (0.2840)        & (0.2837)\\   
   $\delta$         & -0.3333$^{***}$ & -0.2896$^{***}$\\   
   & (0.0635)        & (0.0634)\\   
   $\delta^2$  & 0.0626          & 0.0102\\   
   & (0.0858)        & (0.0857)\\   
   Constant         & 0.0089          & 0.2784$^{***}$\\   
                    & (0.0859)        & (0.0858)\\   
   \midrule
   \emph{Fit statistics}\\
   Observations     & 254             & 254\\  
   R$^2$            & 0.33790         & 0.24751 \\
   \midrule \midrule
   \multicolumn{3}{l}{\emph{Standard-errors in parentheses}}\\
   \multicolumn{3}{l}{\emph{Signif. Codes: ***: 0.01, **: 0.05, *: 0.1}}\\
\end{tabular}
\par\endgroup

            \end{adjustbox}
                \begin{minipage}{1\textwidth}
                    \begin{footnotesize}
                        \begin{singlespace}
                            \emph{Notes:} This table reports the OLS estimates of specification \ref{eq:prop_1_reg} with (Model (2)) and without the equal selection constraint  (Model (2)). The outcome variable is the proportion of women in the finalist pool, $p_h$. The independent variables are the proportion of women in the applicant pool, $p_a$, the correlation between screening and hiring manager scores, $\theta$, and the gender difference in correlation, $\delta$.
                        \end{singlespace}
                    \end{footnotesize}
                \end{minipage}
    \end{table}

    \begin{table}[t] 
    
        \centering
        \caption{Estimated effectiveness of the equal selection constraint}
        \label{tab:parametric_simulation_job_cat}
        \begin{adjustbox}{max width=\textwidth}
            \begin{threeparttable}
                \begin{tabular}{llrrr}
    \toprule
                    
    Job Category                        & Equal Selection & Applied $p_a$                                      & Screened $p_s$ & Hired $p_h$ \\
    \midrule
    Biz Dev \& Operations               & False             & 0.38                                          & 0.38      & 0.38     \\
                                        & True              & 0.38                                          & 0.50      & 0.45     \\
    \midrule
    Customer Service \& Acct Management & False             & 0.38                                          & 0.38      & 0.38     \\
                                        & True              & 0.38                                          & 0.50      & 0.47     \\
    \midrule
    Engineering \& Technical            & False             & 0.23                                          & 0.23      & 0.23     \\
                                        & True              & 0.23                                          & 0.50      & 0.36     \\
    \midrule
    Finance \& Accounting               & False             & 0.35                                          & 0.35      & 0.35     \\
                                        & True              & 0.35                                          & 0.50      & 0.45     \\
    \midrule
    HR                                  & False             & 0.41                                          & 0.41      & 0.41     \\
                                        & True              & 0.41                                          & 0.50      & 0.46     \\
    \midrule
    Legal \& PR                         & False             & 0.35                                          & 0.35      & 0.35     \\
                                        & True              & 0.35                                          & 0.50      & 0.44     \\
    \midrule
    Product \& Design                   & False             & 0.33                                          & 0.33      & 0.33     \\
                                        & True              & 0.33                                          & 0.50      & 0.43     \\
    \midrule
    Sales \& Marketing                  & False             & 0.36                                          & 0.36      & 0.36     \\
                                        & True              & 0.36                                          & 0.50      & 0.44     \\
    \midrule
    Overall                             & False             & 0.31                                          & 0.31      & 0.31     \\
                                        & True              & 0.31                                          & 0.50      & 0.41     \\
    \bottomrule
\end{tabular}

                \begin{tablenotes}[flushleft]
                    \footnotesize
                    \item \emph{Notes:} This table reports the proportion of women in the applicant pool $p_a$, shortlist $p_s$ and hired pool $p_h$ --- with and without the equal selection constraint. $p_a$ is observed in the data. $p_s$ and $p_h$ are estimated by first estimating the job-specific model parameters ($\hat{\theta}_j, \hat{\delta}_j$), and imputing the model parameters into the theoretical model. We estimate at the job posting level and aggregate it up to the job category level for all jobs with $p_a<0.5$ in the hold-out test set.
                \end{tablenotes}
            \end{threeparttable}
        \end{adjustbox}
    \end{table}

    \subsection{Benchmarking Screening Algorithms and Fairness Constraints}
    \label{ssec:benchmarking}
    
    We benchmark the equal selection constraint against other fairness criteria and our proposed complementary screening approach. 
    Specifically, we evaluate seven screening algorithms:

    \begin{enumerate}
    \singlespacing
    \item \textsc{No Constraint}: Baseline without fairness constraints.
    \item \textsc{Equal Selection}: Equal gender representation in the shortlist.
    \item \textsc{Demographic Parity}: Shortlist gender proportions match applicant pool proportions.
    \item \textsc{Error Rate Parity}: Equal error rates for men and women.\footnote{We use fairlearn's implementation \url{https://fairlearn.org/v0.10/user\_guide/mitigation/reductions.html}}
    \item \textsc{Equalized Odds}: Equal true and false positive rates across genders.
    \item \textsc{Equal Selection min $Q^S$ Diff}: Equal selection while minimizing gender differences in screening scores.
    \item \textsc{Complementary Equal Selection}: Equal selection while minimizing gender differences in hiring manager scores.
    \end{enumerate}

    \begin{table}[t]
        \centering
        \caption{Screening algorithms}
        \label{tab:screening_algos}
                \begin{threeparttable}
            \renewcommand{\arraystretch}{1.5}
            \begin{tabular}{m{0.4\textwidth}>{\centering\arraybackslash}m{0.55\textwidth}}
                \toprule
                \textbf{Screening Algorithm}             & \textbf{Constraint}                                                                                    \\
                \hline
                \textsc{No Constraint}                   & None                                                                                                   \\
                \hline
                \textsc{Equal Selection}                 & $\begin{aligned}
                                                                    \mathds{P}(g=f|\hat{y}^S=1) = \mathds{P}(g=m|\hat{y}^S=1)
                                                                \end{aligned}$                                              \\
                \hline
                \textsc{Demographic Parity}              & $\begin{aligned}
                                                                    \mathds{P}(\hat{y}^S|g=f) = \mathds{P}(\hat{y}^S|g=m)
                                                                \end{aligned}$                                                  \\
                \hline
                \textsc{Error Rate Parity}               & $\begin{aligned}
                                                                    \mathds{P}(\hat{y}^S \neq y^S|g=f) = \mathds{P}(\hat{y}^S \neq y^S|g=m)
                                                                \end{aligned}$                                \\
                \hline
                \textsc{Equalized Odds}                  & $\begin{aligned}
                                                                    \mathds{P}(\hat{y}^S =1|y^S, g=f) = \mathds{P}(\hat{y}^S =1|y^S, g=m), \\
                                                                    y^S \in \{0,1\}
                                                                \end{aligned}$                                \\
                \hline

                \textsc{Equal Selection min $Q^S$ Diff.} & $\begin{aligned}
                                                                    \mathds{P}(g=f|\hat{y}^S=1) = \mathds{P}(g=m|\hat{y}^S=1) \\
                                                                    \min \mathds{E}[\hat{q}^S_s|g=f] - \mathds{E}[\hat{q}^S_s|g=m]
                                                                \end{aligned}$ \\
                \hline
                \textsc{Complementary Equal Selection}   & $\begin{aligned}
                                                                    \mathds{P}(g=f|\hat{y}^S=1) = \mathds{P}(g=m|\hat{y}^S=1) \\
                                                                    \min \mathds{E}[\hat{q}^H_s|g=f] - \mathds{E}[\hat{q}^H_s|g=m]
                                                                \end{aligned}$ \\
                \bottomrule
            \end{tabular}
            \begin{tablenotes}
                \footnotesize
                \item\textit{Notes:} This table summarizes the screening algorithms used for benchmarking. $\hat{y}^S$ is the predicted screening outcome, and $y^S$ is the true screening outcome observed in the data. $\hat{q}^S_s$ is the predicted screening score of the shortlisted candidates, and $\hat{q}^H_s$ is the predicted hiring manager score of the shortlisted candidates. The primary objective of all the algorithms is to shortlist candidates with the highest screening score.
            \end{tablenotes}
        \end{threeparttable}
    \end{table}

    We assess the impact of each screening algorithm on both diversity and expected hire quality through agent-based hiring simulations. 
    In these simulations, ML models for screening and hiring manager evaluations serve as agents.
    The screening agent shortlists the candidates with the highest screening scores, $\hat{q}^S$, while satisfying the constraint outlined in \Cref{tab:screening_algos}, and the hiring manager agent selects the candidates with the highest hiring manager scores, $\hat{q}^H$.
    Unlike the theoretical model, these simulations do not assume identical quality distributions for men and women, unbiased hiring managers, or normal distributions, making them robust to potential real-world deviations.

    \paragraph{Simulating true quality $Q$.} To measure the expected hire quality (unobservable in actual data), we generate semi-synthetic quality scores for candidates based on different assumed values of $\theta^S$ and $\theta^H$, while holding the empirically estimated correlation $\theta$ fixed.\footnote{
        For each job we observe the vectors $\bm{q}_j^S$ and $\bm{q}_j^H$ , which fixes $\theta$.
        To generate $\bm{q}$, we first generate a random vector.
        We then orthogonalize it with respect to $\bm{q}_j^S$ and $\bm{q}_j^H$.
        We then transform the vector for a given value for $\theta^S$ and $\theta^H$.
        This produces a random vector $\bm{q}_j$ that has the defined correlation structure $\Sigma$.
    }
    We present the results from these simulations in \Cref{fig:algo_comparison}.

    \begin{figure}[t] 
        \caption{Quality and diversity of hires using different screening algorithms}
        \centering
        \includegraphics[width=\textwidth, keepaspectratio]{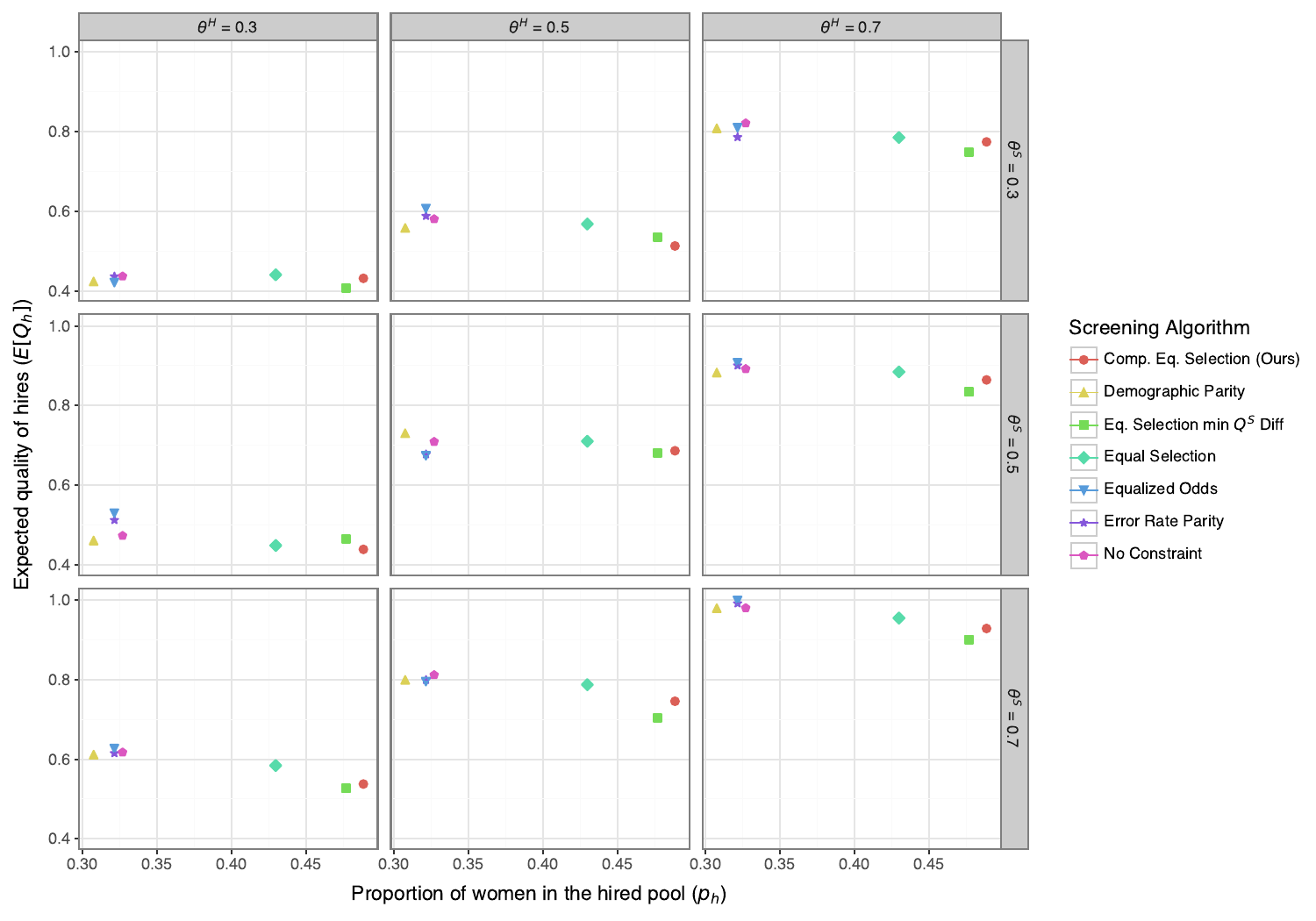}
        \label{fig:algo_comparison}
        \begin{minipage}{1\textwidth}
            \begin{footnotesize}
                \begin{singlespace}
                    \emph{Notes:}
                    This figure plots the expected quality of hires (y-axis) and the proportion of women in the hired pool (x-axis) for different screening algorithms using agent-based hiring simulation experiments. We use semi-synthetic data for true quality $Q$. Each grid in the facet corresponds to a different $\theta^S$ and $\theta^H$ value. Error bars in both the x and y axes represent bootstrapped 95\% 
                    Error bars are not visible because they are narrow.
                \end{singlespace}
            \end{footnotesize}
        \end{minipage}
    \end{figure}

    Our results demonstrate significant variability in the diversity outcomes among the algorithms:

    \begin{itemize}
    
    \item The \textsc{Complementary Equal Selection} algorithm consistently achieves the highest diversity of hires, surpassing all other constraints, including the \textsc{Equal Selection min $Q^S$ Diff.} This performance advantage arises because it directly aligns shortlisted candidates to be complementary to the hiring manager's evaluation criteria, ensuring greater diversity in the actual hires.
    
    \item Algorithms based on \textsc{Demographic Parity}, \textsc{Error Rate Parity}, and \textsc{Equalized Odds} do not substantially enhance diversity compared to the baseline. \textsc{Demographic Parity} actually reduces diversity since, empirically, women had higher shortlisting rates in our training data (see \Cref{apx:callback_regression}), a disparity eliminated by enforcing equal representation with the applicant pool. Likewise, \textsc{Error Rate Parity} and \textsc{Equalized Odds} have limited effect because our ML models exhibit minimal to no gender differences in error rates or ROC curves (as confirmed in \Cref{apx:callback_regression} and \Cref{apx:clf_models}).
    
    \item While the \textsc{Complementary Equal Selection} approach significantly improves diversity, it does so with minimal reduction in the expected quality of hires, particularly under lower values of $\theta^S$ and $\theta^H$.
    
    \end{itemize}

    \section{Discussion and Conclusion}
    \label{sec:discussion_conclusion}
    
    This paper examines the effectiveness of diversity policies implemented as algorithmic fairness constraints within Human+AI hiring systems. We develop a theoretical model of the hiring process, showing that the success of a common diversity policy—equal selection in the shortlist—is contingent upon key parameters such as the correlation between the screening algorithm's and the hiring manager's assessment criteria. Using real hiring data from technology firms, we empirically estimate these parameters and evaluate the impact of equal selection through counterfactual policy simulations.
    
    Our findings indicate that enforcing equal selection in shortlists modestly improves gender diversity among hires but does not achieve parity.
    Moreover, the effectiveness of this constraint varies significantly across job categories.
    To address these limitations, we propose a complementary screening algorithm, designed explicitly to differ from the hiring manager’s assessments, and demonstrate its superior performance in enhancing workforce diversity compared to traditional fairness constraints.
    
    We highlight several critical managerial and algorithmic design implications arising from our results:

    \begin{itemize}
    \item First, achieving gender parity at the shortlist stage does not inherently guarantee gender parity in final hires, even if hiring managers are gender-unbiased. Without this understanding, stakeholders may mistakenly interpret the post-shortlist disparities as biases introduced by hiring managers, undermining trust in the fairness policy.
    
    \item Second, the effectiveness of the equal selection constraint is highly job-specific, driven by the correlation between screener and hiring manager evaluations. Notably, technical roles requiring measurable ``hard skills'' (e.g., software engineering) tend to exhibit higher correlations, diminishing the effectiveness of equal selection precisely in fields where women are most underrepresented.
    
    \item Third, equal predictive accuracy of screening algorithms across genders is insufficient in multi-stage hiring processes. It is equally important for screening algorithms to maintain gender neutrality concerning their alignment with hiring managers’ criteria—specifically, algorithms should exhibit no gender differences in their correlation with managerial assessments ($\delta = 0$).
    
    \item Lastly, we show theoretically that higher correlations between screeners' and hiring managers' assessments not only reduce the effectiveness of equal selection constraints but also negatively affect the expected quality of hires. This suggests a critical design insight: screening algorithms should be constructed to complement, rather than replicate, managerial evaluations.
    \end{itemize}
    
    \subsubsection*{Limitations and Future Directions}
    
    Our recommendation—that screening algorithms should complement hiring managers’ assessments-may encounter practical organizational challenges. Hiring managers often perceive AI tools as substitutes to replicate their decision-making processes. Previous literature has identified similar tensions in algorithmic hiring contexts \parencite{van_den_broek_when_2021}. Future research could explore strategies for effectively integrating AI tools explicitly designed to complement, rather than duplicate, human judgment.
    
    Another limitation of our model is the assumption that hiring managers do not adapt their decision-making in response to diversity constraints. Future studies should investigate whether introducing fairness constraints could inadvertently induce biases among previously unbiased hiring managers. Existing psychology and management literature highlights that affirmative action programs (AAPs) can lead to stigma and stereotyping of minority candidates, even those not directly benefiting from AAPs \parencite{heilman_affirmative_1997,leslie_stigma_2013}. Given that algorithmic fairness constraints might be perceived similarly to AAPs, further research is needed to understand such policies' potential psychological and organizational impacts.

    \singlespacing
    \printbibliography

@inproceedings{blum_multi_2022,
  title      = {Multi {{Stage Screening}}: {{Enforcing Fairness}} and {{Maximizing Efficiency}} in a {{Pre-Existing Pipeline}}},
  shorttitle = {Multi {{Stage Screening}}},
  booktitle  = {2022 {{ACM Conference}} on {{Fairness}}, {{Accountability}}, and {{Transparency}}},
  author     = {Blum, Avrim and Stangl, Kevin and Vakilian, Ali},
  date       = {2022-06-20},
  series     = {{{FAccT}} '22},
  pages      = {1178--1193},
  publisher  = {{Association for Computing Machinery}},
  location   = {{New York, NY, USA}},
  doi        = {10.1145/3531146.3533178},
  url        = {https://doi.org/10.1145/3531146.3533178},
  urldate    = {2023-02-22},
  isbn       = {978-1-4503-9352-2}
}

@article{brands_leaning_2017,
  title        = {Leaning {{Out}}: {{How Negative Recruitment Experiences Shape Women}}’s {{Decisions}} to {{Compete}} for {{Executive Roles}}},
  shorttitle   = {Leaning {{Out}}},
  author       = {Brands, Raina A. and Fernandez-Mateo, Isabel},
  date         = {2017-09-01},
  journaltitle = {Administrative Science Quarterly},
  shortjournal = {Administrative Science Quarterly},
  volume       = {62},
  number       = {3},
  pages        = {405--442},
  publisher    = {{SAGE Publications Inc}},
  issn         = {0001-8392},
  doi          = {10.1177/0001839216682728},
  url          = {https://doi.org/10.1177/0001839216682728},
  urldate      = {2020-12-23},
  langid       = {english}
}

@inproceedings{celis_effect_2021,
  title     = {The {{Effect}} of the {{Rooney Rule}} on {{Implicit Bias}} in the {{Long Term}}},
  booktitle = {Proceedings of the 2021 {{ACM Conference}} on {{Fairness}}, {{Accountability}}, and {{Transparency}}},
  author    = {Celis, L. Elisa and Hays, Chris and Mehrotra, Anay and Vishnoi, Nisheeth K.},
  date      = {2021-03-03},
  series    = {{{FAccT}} '21},
  pages     = {678--689},
  publisher = {{Association for Computing Machinery}},
  location  = {{New York, NY, USA}},
  doi       = {10.1145/3442188.3445930},
  url       = {https://doi.org/10.1145/3442188.3445930},
  urldate   = {2021-06-21},
  isbn      = {978-1-4503-8309-7}
}

@article{chouldechova_fair_2017,
  title        = {Fair {{Prediction}} with {{Disparate Impact}}: {{A Study}} of {{Bias}} in {{Recidivism Prediction Instruments}}},
  shorttitle   = {Fair {{Prediction}} with {{Disparate Impact}}},
  author       = {Chouldechova, Alexandra},
  date         = {2017-06-01},
  journaltitle = {Big Data},
  volume       = {5},
  number       = {2},
  pages        = {153--163},
  publisher    = {{Mary Ann Liebert, Inc., publishers}},
  issn         = {2167-6461},
  doi          = {10.1089/big.2016.0047},
  url          = {https://www.liebertpub.com/doi/full/10.1089/big.2016.0047},
  urldate      = {2021-05-27}
}

@unpublished{cowgill_bias_2020,
  title    = {Bias and {{Productivity}} in {{Humans}} and {{Algorithms}}: {{Theory}} and {{Evidence}} from {{Re}}´sume´ {{Screening}}},
  author   = {Cowgill, Bo},
  date     = {2020},
  langid   = {english}
}

@inproceedings{dwork_fairness_2012,
  ids       = {dwork_fairness_2012-1},
  title     = {Fairness through Awareness},
  booktitle = {Proceedings of the 3rd {{Innovations}} in {{Theoretical Computer Science Conference}}},
  author    = {Dwork, Cynthia and Hardt, Moritz and Pitassi, Toniann and Reingold, Omer and Zemel, Richard},
  date      = {2012-01-08},
  series    = {{{ITCS}} '12},
  pages     = {214--226},
  publisher = {{Association for Computing Machinery}},
  location  = {{New York, NY, USA}},
  doi       = {10.1145/2090236.2090255},
  url       = {https://doi.org/10.1145/2090236.2090255},
  urldate   = {2021-05-16},
  isbn      = {978-1-4503-1115-1}
}

@article{fershtman_soft_2021,
  title        = {“{{Soft}}” {{Affirmative Action}} and {{Minority Recruitment}}},
  author       = {Fershtman, Daniel and Pavan, Alessandro},
  date         = {2021-03-01},
  journaltitle = {American Economic Review: Insights},
  shortjournal = {American Economic Review: Insights},
  volume       = {3},
  number       = {1},
  pages        = {1--18},
  issn         = {2640-205X, 2640-2068},
  doi          = {10.1257/aeri.20200196},
  url          = {https://pubs.aeaweb.org/doi/10.1257/aeri.20200196},
  urldate      = {2022-04-16},
  langid       = {english}
}

@unpublished{geyik_fairness-aware_2019,
  title       = {Fairness-{{Aware Ranking}} in {{Search}} \& {{Recommendation Systems}} with {{Application}} to {{LinkedIn Talent Search}}},
  author      = {Geyik, Sahin Cem and Ambler, Stuart and Kenthapadi, Krishnaram},
  date        = {2019-04-30},
  eprint      = {1905.01989},
  eprinttype  = {arxiv},
  eprintclass = {cs},
  url         = {http://arxiv.org/abs/1905.01989},
  urldate     = {2019-08-22}
}

@unpublished{hardt_equality_2016,
  title       = {Equality of {{Opportunity}} in {{Supervised Learning}}},
  author      = {Hardt, Moritz and Price, Eric and Srebro, Nathan},
  date        = {2016-10-07},
  eprint      = {1610.02413},
  eprinttype  = {arxiv},
  eprintclass = {cs},
  url         = {http://arxiv.org/abs/1610.02413},
  urldate     = {2021-05-28}
}

@article{heilman_affirmative_1997,
  title        = {The {{Affirmative Action Stigma Of Incompetence}}: {{Effects Of Performance Information Ambiguity}}},
  shorttitle   = {The {{Affirmative Action Stigma Of Incompetence}}},
  author       = {Heilman, Madeline E. and Block, Caryn J. and Stathatos, Peter},
  date         = {1997-06-01},
  journaltitle = {Academy of Management Journal},
  shortjournal = {AMJ},
  volume       = {40},
  number       = {3},
  pages        = {603--625},
  issn         = {0001-4273},
  doi          = {10.5465/257055},
  url          = {https://journals.aom.org/doi/abs/10.5465/257055},
  urldate      = {2019-12-08}
}

@article{huet_facebooks_2017,
  entrysubtype = {newspaper},
  title        = {Facebook’s {{Hiring Process Hinders Its Effort}} to {{Create}} a {{Diverse Workforce}}},
  author       = {Huet, Ellen},
  date         = {2017-01-10},
  url          = {https://www.bloombergquint.com/technology/facebook-s-hiring-process-hinders-its-effort-to-create-a-diverse-workforce},
  urldate      = {2020-12-06}
}

@book{joe_dependence_2014,
  title      = {Dependence {{Modeling}} with {{Copulas}}},
  author     = {Joe, Harry},
  date       = {2014-06-26},
  eprint     = {09ThAwAAQBAJ},
  eprinttype = {googlebooks},
  publisher  = {{CRC Press}},
  isbn       = {978-1-4665-8322-1},
  langid     = {english},
  pagetotal  = {483}
}

@unpublished{kleinberg_inherent_2016,
  title       = {Inherent {{Trade-Offs}} in the {{Fair Determination}} of {{Risk Scores}}},
  author      = {Kleinberg, Jon and Mullainathan, Sendhil and Raghavan, Manish},
  date        = {2016-11-17},
  eprint      = {1609.05807},
  eprinttype  = {arxiv},
  eprintclass = {cs, stat},
  url         = {http://arxiv.org/abs/1609.05807},
  urldate     = {2021-05-27}
}

@unpublished{kleinberg_selection_2018,
  title       = {Selection {{Problems}} in the {{Presence}} of {{Implicit Bias}}},
  author      = {Kleinberg, Jon and Raghavan, Manish},
  date        = {2018-01-04},
  eprint      = {1801.03533},
  eprinttype  = {arxiv},
  eprintclass = {cs, stat},
  url         = {http://arxiv.org/abs/1801.03533},
  urldate     = {2021-06-03}
}

@article{lee_diversity_2021,
  title        = {Diversity and the Timing of Preference in Hiring Decisions},
  author       = {Lee, Logan M. and Waddell, Glen R.},
  date         = {2021-04-01},
  journaltitle = {Journal of Economic Behavior \& Organization},
  shortjournal = {Journal of Economic Behavior \& Organization},
  volume       = {184},
  pages        = {432--459},
  issn         = {0167-2681},
  doi          = {10.1016/j.jebo.2020.11.014},
  url          = {https://www.sciencedirect.com/science/article/pii/S0167268120304169},
  urldate      = {2023-02-22},
  langid       = {english}
}

@article{leslie_stigma_2013,
  title        = {The {{Stigma}} of {{Affirmative Action}}: {{A Stereotyping-Based Theory}} and {{Meta-Analytic Test}} of the {{Consequences}} for {{Performance}}},
  shorttitle   = {The {{Stigma}} of {{Affirmative Action}}},
  author       = {Leslie, Lisa M. and Mayer, David M. and Kravitz, David A.},
  date         = {2013-07-23},
  journaltitle = {Academy of Management Journal},
  shortjournal = {AMJ},
  volume       = {57},
  number       = {4},
  pages        = {964--989},
  issn         = {0001-4273},
  doi          = {10.5465/amj.2011.0940},
  url          = {https://journals.aom.org/doi/abs/10.5465/amj.2011.0940},
  urldate      = {2019-12-08}
}

@article{mitchell_algorithmic_2021,
  title        = {Algorithmic {{Fairness}}: {{Choices}}, {{Assumptions}}, and {{Definitions}}},
  shorttitle   = {Algorithmic {{Fairness}}},
  author       = {Mitchell, Shira and Potash, Eric and Barocas, Solon and D'Amour, Alexander and Lum, Kristian},
  date         = {2021},
  journaltitle = {Annual Review of Statistics and Its Application},
  volume       = {8},
  number       = {1},
  pages        = {141--163},
  doi          = {10.1146/annurev-statistics-042720-125902},
  url          = {https://doi.org/10.1146/annurev-statistics-042720-125902},
  urldate      = {2021-05-27}
}

@book{nelsen_introduction_2007,
  title      = {An {{Introduction}} to {{Copulas}}},
  author     = {Nelsen, Roger B.},
  date       = {2007-06-10},
  eprint     = {yexFAAAAQBAJ},
  eprinttype = {googlebooks},
  publisher  = {{Springer Science \& Business Media}},
  isbn       = {978-0-387-28678-5},
  langid     = {english},
  pagetotal  = {277}
}

@article{peng_what_2019,
  title        = {What {{You See Is What You Get}}? {{The Impact}} of {{Representation Criteria}} on {{Human Bias}} in {{Hiring}}},
  shorttitle   = {What {{You See Is What You Get}}?},
  author       = {Peng, Andi and Nushi, Besmira and Kıcıman, Emre and Inkpen, Kori and Suri, Siddharth and Kamar, Ece},
  date         = {2019-10-28},
  journaltitle = {Proceedings of the AAAI Conference on Human Computation and Crowdsourcing},
  volume       = {7},
  number       = {1},
  pages        = {125--134},
  url          = {https://wvvw.aaai.org/ojs/index.php/HCOMP/article/view/5281},
  urldate      = {2020-04-23},
  issue        = {1},
  langid       = {english}
}

@inproceedings{raghavan_mitigating_2020,
  ids        = {raghavan_mitigating_2020-1},
  title      = {Mitigating Bias in Algorithmic Hiring: Evaluating Claims and Practices},
  shorttitle = {Mitigating Bias in Algorithmic Hiring},
  booktitle  = {Proceedings of the 2020 {{Conference}} on {{Fairness}}, {{Accountability}}, and {{Transparency}}},
  author     = {Raghavan, Manish and Barocas, Solon and Kleinberg, Jon and Levy, Karen},
  date       = {2020-01-27},
  series     = {{{FAT}}* '20},
  pages      = {469--481},
  publisher  = {{Association for Computing Machinery}},
  location   = {{New York, NY, USA}},
  doi        = {10.1145/3351095.3372828},
  url        = {https://doi.org/10.1145/3351095.3372828},
  urldate    = {2021-02-21},
  isbn       = {978-1-4503-6936-7}
}

@article{schuck_affirmative_2002,
  title        = {Affirmative {{Action}}: {{Past}}, {{Present}}, and {{Future}}},
  shorttitle   = {Affirmative {{Action}}},
  author       = {Schuck, Peter H.},
  date         = {2002},
  journaltitle = {Yale Law \& Policy Review},
  shortjournal = {Yale L. \& Pol'y Rev.},
  volume       = {20},
  number       = {1},
  pages        = {1--96},
  url          = {https://heinonline.org/HOL/P?h=hein.journals/yalpr20&i=7},
  urldate      = {2022-10-04},
  langid       = {english}
}

@article{shi_adoption_2018,
  title        = {The Adoption of Chief Diversity Officers among {{S}}\&{{P}} 500 Firms: {{Institutional}}, Resource Dependence, and Upper Echelons Accounts},
  shorttitle   = {The Adoption of Chief Diversity Officers among {{S}}\&{{P}} 500 Firms},
  author       = {Shi, Wei and Pathak, Seemantini and Song, Lynda Jiwen and Hoskisson, Robert E.},
  date         = {2018},
  journaltitle = {Human Resource Management},
  volume       = {57},
  number       = {1},
  pages        = {83--96},
  issn         = {1099-050X},
  doi          = {10.1002/hrm.21837},
  url          = {https://onlinelibrary.wiley.com/doi/abs/10.1002/hrm.21837},
  urldate      = {2019-12-07},
  langid       = {english}
}

@article{storvik_search_2008,
  title        = {In Search of the Glass Ceiling: Gender and Recruitment to Management in {{Norway}}'s State Bureaucracy1},
  shorttitle   = {In Search of the Glass Ceiling},
  author       = {Storvik, Aagoth Elise and Schøne, Pål},
  date         = {2008},
  journaltitle = {The British Journal of Sociology},
  volume       = {59},
  number       = {4},
  pages        = {729--755},
  issn         = {1468-4446},
  doi          = {10.1111/j.1468-4446.2008.00217.x},
  url          = {https://onlinelibrary.wiley.com/doi/abs/10.1111/j.1468-4446.2008.00217.x},
  urldate      = {2021-02-22},
  langid       = {english}
}

@inproceedings{suhr_does_2021,
  title      = {Does {{Fair Ranking Improve Minority Outcomes}}? {{Understanding}} the {{Interplay}} of {{Human}} and {{Algorithmic Biases}} in {{Online Hiring}}},
  shorttitle = {Does {{Fair Ranking Improve Minority Outcomes}}?},
  booktitle  = {Proceedings of the 2021 {{AAAI}}/{{ACM Conference}} on {{AI}}, {{Ethics}}, and {{Society}}},
  author     = {Sühr, Tom and Hilgard, Sophie and Lakkaraju, Himabindu},
  date       = {2021-07-21},
  series     = {{{AIES}} '21},
  pages      = {989--999},
  publisher  = {{Association for Computing Machinery}},
  location   = {{New York, NY, USA}},
  doi        = {10.1145/3461702.3462602},
  url        = {https://doi.org/10.1145/3461702.3462602},
  urldate    = {2021-09-10},
  isbn       = {978-1-4503-8473-5}
}

@inproceedings{sun_how_2019,
  title     = {How to {{Fine-Tune BERT}} for {{Text Classification}}?},
  booktitle = {Chinese {{Computational Linguistics}}},
  author    = {Sun, Chi and Qiu, Xipeng and Xu, Yige and Huang, Xuanjing},
  editor    = {Sun, Maosong and Huang, Xuanjing and Ji, Heng and Liu, Zhiyuan and Liu, Yang},
  date      = {2019},
  series    = {Lecture {{Notes}} in {{Computer Science}}},
  pages     = {194--206},
  publisher = {{Springer International Publishing}},
  location  = {{Cham}},
  doi       = {10.1007/978-3-030-32381-3_16},
  isbn      = {978-3-030-32381-3},
  langid    = {english}
}

@article{van_den_broek_when_2021,
  title        = {When the Machine Meets the Expert: {{An}} Ethnography of Developing {{AI}} for Hiring},
  shorttitle   = {When the Machine Meets the Expert},
  author       = {family=Broek, given=Elmira, prefix=van den, useprefix=true and Sergeeva, Anastasia and Huysman, Marleen},
  date         = {2021-09},
  journaltitle = {MIS Quarterly},
  volume       = {45},
  number       = {3},
  pages        = {1557--1580},
  issn         = {0276-7783},
  doi          = {10.25300/MISQ/2021/16559},
  url          = {http://www.scopus.com/inward/record.url?scp=85114672912&partnerID=8YFLogxK},
  urldate      = {2022-08-10}
}

@inproceedings{zafar_fairness_2017,
  title      = {Fairness {{Beyond Disparate Treatment}} \&amp; {{Disparate Impact}}: {{Learning Classification}} without {{Disparate Mistreatment}}},
  shorttitle = {Fairness {{Beyond Disparate Treatment}} \&amp; {{Disparate Impact}}},
  booktitle  = {Proceedings of the 26th {{International Conference}} on {{World Wide Web}}},
  author     = {Zafar, Muhammad Bilal and Valera, Isabel and Gomez Rodriguez, Manuel and Gummadi, Krishna P.},
  date       = {2017-04-03},
  series     = {{{WWW}} '17},
  pages      = {1171--1180},
  publisher  = {{International World Wide Web Conferences Steering Committee}},
  location   = {{Republic and Canton of Geneva, CHE}},
  doi        = {10.1145/3038912.3052660},
  url        = {https://doi.org/10.1145/3038912.3052660},
  urldate    = {2021-06-03},
  isbn       = {978-1-4503-4913-0}
}

@unpublished{zafar_parity_2017,
  title       = {From {{Parity}} to {{Preference-based Notions}} of {{Fairness}} in {{Classification}}},
  author      = {Zafar, Muhammad Bilal and Valera, Isabel and Rodriguez, Manuel Gomez and Gummadi, Krishna P. and Weller, Adrian},
  date        = {2017-11-28},
  eprint      = {1707.00010},
  eprinttype  = {arxiv},
  eprintclass = {cs, stat},
  url         = {http://arxiv.org/abs/1707.00010},
  urldate     = {2021-06-03}
}

@inproceedings{zaheer_big_2020,
  title      = {Big {{Bird}}: {{Transformers}} for {{Longer Sequences}}},
  shorttitle = {Big {{Bird}}},
  author     = {Zaheer, Manzil and Guruganesh, Guru Prashanth and Dubey, Avi and Ainslie, Joshua and Alberti, Chris and Ontanon, Santiago and Pham, Philip Minh and Ravula, Anirudh and Wang, Qifan and Yang, Li and Ahmed, Amr Mahmoud El Houssieny},
  date       = {2020},
  url        = {https://proceedings.neurips.cc/paper/2020/file/c8512d142a2d849725f31a9a7a361ab9-Paper.pdf},
  urldate    = {2022-04-10},
  eventtitle = {Neural {{Information Processing Systems}} ({{NeurIPS}})}
}

@inproceedings{zemel_learning_2013,
  title      = {Learning {{Fair Representations}}},
  booktitle  = {International {{Conference}} on {{Machine Learning}}},
  author     = {Zemel, Rich and Wu, Yu and Swersky, Kevin and Pitassi, Toni and Dwork, Cynthia},
  date       = {2013-05-26},
  pages      = {325--333},
  publisher  = {{PMLR}},
  issn       = {1938-7228},
  url        = {http://proceedings.mlr.press/v28/zemel13.html},
  urldate    = {2021-05-30},
  eventtitle = {International {{Conference}} on {{Machine Learning}}},
  langid     = {english}
}

@article{clemen_limits_1985,
  title     = {Limits for the {{Precision}} and {{Value}} of {{Information}} from {{Dependent Sources}}},
  author    = {Clemen, Robert T. and Winkler, Robert L.},
  year      = {1985},
  month     = apr,
  journal   = {Operations Research},
  volume    = {33},
  number    = {2},
  pages     = {427--442},
  publisher = {INFORMS},
  issn      = {0030-364X},
  doi       = {10.1287/opre.33.2.427},
  urldate   = {2024-07-12}
}

@inproceedings{khalili_fair_2021,
  title     = {Fair {{Sequential Selection Using Supervised Learning Models}}},
  booktitle = {Advances in {{Neural Information Processing Systems}}},
  author    = {Khalili, Mohammad Mahdi and Zhang, Xueru and Abroshan, Mahed},
  year      = {2021},
  volume    = {34},
  pages     = {28144--28155},
  publisher = {Curran Associates, Inc.},
  urldate   = {2024-09-03}
}

@inproceedings{bower_fair_2017,
  title         = {Fair {{Pipelines}}},
  booktitle     = {Workshop on {{Fairness}}, {{Accountability}}, and {{Transparency}} in {{Machine Learning}}},
  author        = {Bower, Amanda and Kitchen, Sarah N. and Niss, Laura and Strauss, Martin J. and Vargas, Alexander and Venkatasubramanian, Suresh},
  year          = {2017},
  month         = jul,
  eprint        = {1707.00391},
  primaryclass  = {cs, stat},
  publisher     = {arXiv},
  doi           = {10.48550/arXiv.1707.00391},
  urldate       = {2024-09-03},
  archiveprefix = {arXiv}
}

@article{dwork_fairness_2019,
  title         = {Fairness {{Under Composition}}},
  author        = {Dwork, Cynthia and Ilvento, Christina},
  year          = {2019},
  journal       = {LIPIcs, Volume 124, ITCS 2019},
  volume        = {124},
  eprint        = {1806.06122},
  primaryclass  = {cs, stat},
  pages         = {33:1-33:20},
  issn          = {1868-8969},
  doi           = {10.4230/LIPIcs.ITCS.2019.33},
  urldate       = {2024-09-03},
  archiveprefix = {arXiv}
}

@misc{bapna_rejection_2021,
  type    = {{{SSRN Scholarly Paper}}},
  title   = {Rejection {{Communication}} and {{Women}}'s {{Job-Search Persistence}}},
  author  = {Bapna, Sofia and Benson, Alan and Funk, Russell},
  year    = {2021},
  month   = oct,
  number  = {3953695},
  address = {Rochester, NY},
  doi     = {10.2139/ssrn.3953695},
  urldate = {2024-01-29}
}

@inproceedings{jiang_fair_2023,
  title     = {Fair {{Selection}} through {{Kernel Density Estimation}}},
  booktitle = {2023 {{International Joint Conference}} on {{Neural Networks}} ({{IJCNN}})},
  author    = {Jiang, Xiangyu and Dai, Yucong and Wu, Yongkai},
  year      = {2023},
  month     = jun,
  pages     = {1--8},
  issn      = {2161-4407},
  doi       = {10.1109/IJCNN54540.2023.10191616},
  urldate   = {2024-09-04}
}

@misc{rooney_rule_nfl,
  author       = {{NFL Operations}},
  title        = {The Rooney Rule},
  howpublished = {\url{https://operations.nfl.com/inside-football-ops/diversity-inclusion/the-rooney-rule/}},
  note         = {Accessed: 2025-04-01},
  year         = 2003
}

@misc{patreon_diversity,
  author       = {{Patreon}},
  title        = {Patreon Culture Deck},
  howpublished = {\url{https://www.slideshare.net/TarynArnold/patreon-culture-deck-april-2017}},
  note         = {Accessed: 2025-04-01},
  year         = 2017
}

@misc{pinterest_diversity,
  author       = {{Pinterest}},
  title        = {Our plan for a more diverse Pinterest},
  howpublished = {\url{https://newsroom-archive.pinterest.com/our-plan-for-a-more-diverse-pinterest-2015}},
  note         = {Accessed: 2025-04-01},
  year         = 2015
}

@misc{facebook_diversity,
  author       = {{Facebook}},
  title        = {Facebook Diversity Efforts},
  howpublished = {\url{https://about.fb.com/news/2021/07/facebook-diversity-report-2021/}},
  note         = {Accessed: 2025-04-01},
  year         = 2021
}

@article{tallis1961moment,
  title     = {The moment generating function of the truncated multi-normal distribution},
  author    = {Tallis, Georges M},
  journal   = {Journal of the Royal Statistical Society Series B: Statistical Methodology},
  volume    = {23},
  number    = {1},
  pages     = {223--229},
  year      = {1961},
  publisher = {Oxford University Press}
}

@book{tong2012multivariate,
  title     = {The multivariate normal distribution},
  author    = {Tong, Yung Liang},
  year      = {2012},
  publisher = {Springer Science \& Business Media}
}
    
    \section*{Appendix}
    \doublespacing
    \appendix
    
\section{Proofs}
\label{apx:proofs}

\newcommand{\ph}{\phi}
\newcommand{\bPh}{\overline{\Phi}}
\newcommand{\bPhTwo}{\overline{\Phi}_2}

\subsection{Setup, Definitions, and Key Identities}
\label{apx:preliminaries}
We use the following definitions and identities throughout the proofs.

\paragraph{Setup} 
Applicants are evaluated using the screening score ($Q^S$) in the first stage, and the hiring score ($Q^H$) in the second.
    Unless otherwise specified, we assume that $(\delta, \delta^S, \delta^H, \alpha, \beta^S, \beta^H) = 0$; the distribution of the scores are thus the same for men and women before any selection.
    \begin{align}
        (Q, Q^S, Q^H) &\sim \mathcal{N} \left( \begin{bmatrix} 0 & 0 & 0 \end{bmatrix}, \begin{bmatrix}
        1          & \theta^S & \theta^H \\
        \theta^S & 1          & \theta   \\
        \theta^H & \theta   & 1
        \end{bmatrix} \right)
        \end{align}

\paragraph{Definitions}
\begin{enumerate}
    \item $\phi(x)$ is the standard normal Probability Density Function (PDF). $Q, Q^S, Q^H$ have marginal PDFs $\phi(q), \phi(q^S), \phi(q^H)$.
    \item $\Phi(x)$ is the standard normal Cumulative Distribution Function (CDF). $Q, Q^S, Q^H$ have marginal CDFs $\Phi(q), \Phi(q^S), \Phi(q^H)$.
    \item $\phi_2(x, y; \rho)$ is the standard bivariate normal PDF evaluated at $(x,y)$ with correlation $\rho$. $Q^S, Q^H$ have joint PDF $\phi_2(q^S, q^H; \theta)$.
    \item $\Phi_2(x, y; \rho)$ is the standard bivariate normal CDF, representing $P(X \le x, Y \le y)$ where $(X, Y)$ have standard normal marginal distributions with correlation $\rho$. $Q^S, Q^H$ have joint CDF $\Phi_2(q^S, q^H; \theta)$.
    \item $\bPh_2(x, y; \rho)$ is the complement of the standard bivariate normal CDF (tail distribution), representing $P(X > x, Y > y)$ where $(X, Y)$ have standard normal marginal distributions with correlation $\rho$. The probability that a candidate is hired is therefore given by $\bPh_2(\tau^S, \tau^H; \theta)$.
\end{enumerate}
\paragraph{Key Identities}
\begin{enumerate}
    \item Plackett's Identity for the derivative of the bivariate normal CDF with respect to correlation (see \textcite{tong2012multivariate}):
    $$ \frac{\partial}{\partial \rho} \Phi_2(a, b; \rho) = \phi_2(a, b; \rho) $$
    \item Property of the bivariate normal PDF:
    $$ \phi_2(-x, -y; \rho) = \phi_2(x, y; \rho) $$
    This is because the exponent $-\frac{(-x)^2 - 2\rho (-x)(-y) + (-y)^2}{2(1-\rho^2)} = -\frac{x^2 - 2\rho xy + y^2}{2(1-\rho^2)}$ is unchanged.
    \item Relation between univariate and bivariate normal PDFs:
    $$ \phi(x) \phi\left(\frac{y - \rho x}{\sqrt{1-\rho^2}}\right) = \sqrt{1-\rho^2} \phi_2(x, y; \rho) $$
    And similarly, by symmetry of $\phi_2(x,y;\rho) = \phi_2(y,x;\rho)$:
    $$ \phi(y) \phi\left(\frac{x - \rho y}{\sqrt{1-\rho^2}}\right) = \sqrt{1-\rho^2} \phi_2(x, y; \rho) $$
    \item Symmetry of the univariate normal PDF: $\phi(-x) = \phi(x)$.
    \item Symmetry relating CDF and tail distribution: $\bPh_2(x, y; \rho) = \Phi_2(-x, -y; \rho)$.
\end{enumerate}

\subsection{Proof of Proposition 1}
\label{sec:proof_theta_vs_ph}

\begin{proposition*}
    \cmdphVsTheta{}
\end{proposition*}

\begin{proof}
The goal is to show that the derivative of female proportion of hires, $p_h$, with respect to $\theta$ is negative for $\theta \in [0, 1)$ --- i.e., $\frac{\partial p_h}{\partial\theta} < 0$ for $\theta \in [0, 1)$.

\begin{enumerate}[wide, labelwidth=!, labelindent=0pt, leftmargin=*]
    \item \textbf{Define Female Proportion of Hires, $p_h(\theta)$.} 
    
    Let $N(\theta) = \Pr(\text{female is hired})$ and $M(\theta) = \Pr(\text{male is hired})$. Then:
    \[ p_h(\theta) = \frac{N(\theta)}{N(\theta) + M(\theta)} \]
    \begin{align*}
        N(\theta) &= p_a \int_{\tau_f^S}^{\infty} \Pr(Q^H > \tau^H | Q^S=q^S, \theta) f_{Q^S}(q^S) dq^S \\
        M(\theta) &= (1-p_a) \int_{\tau_m^S}^{\infty} \Pr(Q^H > \tau^H | Q^S=q^S, \theta) f_{Q^S}(q^S) dq^S
    \end{align*}
    
    where $p_a$ is the proportion of women in the applicant pool ($p_a < 0.5$); $\tau_f^S$ and $\tau_m^S$ are the screening thresholds for women and men, respectively; and $\tau^H$ is the hiring threshold.
    $f_{Q^S} = \phi(q^S)$ is the standard normal PDF.
    
    \item \textbf{Shortlisting Thresholds under Equal Selection.}
    Under the \emph{equal selection} constraint, the shortlisting thresholds $\tau_f^S$ for women and $\tau_m^S$ for men are adjusted such that the number of shortlisted women equals the number of shortlisted men:
    \[ p_a \Pr(Q^S > \tau_f^S | \text{female}) = (1-p_a) \Pr(Q^S > \tau_m^S | \text{male}) \]
    Let $F_{Q^S}$ be the CDF of $Q^S$. Then:
    \[ p_a (1 - F_{Q^S}(\tau_f^S)) = (1-p_a) (1 - F_{Q^S}(\tau_m^S)) \]
    Given $p_a < 0.5$, then $1 - F_{Q^S}(\tau_f^S) > 1 - F_{Q^S}(\tau_m^S)$, which means $\tau_f^S < \tau_m^S$. Women face a lower bar for shortlisting.
    In more explicit terms:
    \begin{equation}
    \label{equ:tau_f}
    \tau_f^S = F_{Q^S}^{-1} \left( 1 - \frac{1-p_a}{p_a} (1 - F_{Q^S}(\tau_m^S)) \right)
    \end{equation}

    \item \textbf{Conditional Hiring Probability.} The conditional distribution of hiring scores is $Q^H | Q^S=q^S \sim \mathcal{N}(\theta q^S, 1-\theta^2)$. 
    The probability that a candidate is hired given their screening score is:
    \[ \Pr(Q^H > \tau^H | Q^S=q^S, \theta) = 1 - \Phi\left(\frac{\tau^H - \theta q^S}{\sqrt{1 - \theta^2}}\right) = \Phi\left(\frac{\theta q^S - \tau^H}{\sqrt{1 - \theta^2}}\right) \]


    \item \textbf{Derivative of Conditional Hiring Probability.} Let $P(H|q^S, \theta) = \Pr(Q^H > \tau^H | Q^S=q^S, \theta)$.
    \[ \frac{\partial P(H|q^S, \theta)}{\partial \theta} = \phi\left(\frac{\theta q^S - \tau^H}{\sqrt{1 - \theta^2}}\right) \cdot \frac{d}{d\theta}\left(\frac{\theta q^S - \tau^H}{\sqrt{1 - \theta^2}}\right) \]
    Calculating the derivative of the argument:
    \[ \frac{d}{d\theta}\left(\frac{\theta q^S - \tau^H}{\sqrt{1 - \theta^2}}\right) = \frac{q^S\sqrt{1 - \theta^2} - (\theta q^S - \tau^H) \frac{- \theta}{\sqrt{1 - \theta^2}}}{1 - \theta^2} = \frac{q^S(1 - \theta^2) + \theta(\theta q^S - \tau^H)}{(1 - \theta^2)^{3/2}} = \frac{q^S - \theta \tau^H}{(1 - \theta^2)^{3/2}} \]
    So,
    \[ \frac{\partial P(H|q^S, \theta)}{\partial \theta} = \phi\left(\frac{\theta q^S - \tau^H}{\sqrt{1 - \theta^2}}\right) \frac{q^S - \theta \tau^H}{(1 - \theta^2)^{3/2}} \]
    Since $\phi(\cdot) > 0$ and $(1 - \theta^2)^{3/2} > 0$ for $\theta \in [0, 1)$, the sign depends on $(q^S - \theta \tau^H)$. The derivative is larger for larger values of $q^S$.

    \item \textbf{Derivative of $p_h(\theta)$.} Using the quotient rule, $\frac{d p_h}{d \theta}$ has the same sign as $N'(\theta)(N(\theta)+M(\theta)) - N(\theta)(N'(\theta)+M'(\theta)) = N'(\theta)M(\theta) - N(\theta)M'(\theta)$.
    We want to show this is negative, which is equivalent to showing:
    \[ \frac{N'(\theta)}{N(\theta)} < \frac{M'(\theta)}{M(\theta)} \]
    This means the relative rate of change of the hiring probability with respect to $\theta$ is smaller for women than for men.

    \item \textbf{Connecting to Thresholds.}
    The goal is to understand how the difference in thresholds ($\tau_f^S < \tau_m^S$) interacts with the derivative of the conditional hiring probability to make $p_h(\theta)$ decrease as $\theta$ increases.
    \begin{align*}
        N'(\theta) &= p_a \int_{\tau_f^S}^{\infty} \frac{\partial P(H|q^S, \theta)}{\partial \theta} \phi(q^S) dq^S \\
        M'(\theta) &= (1-p_a) \int_{\tau_m^S}^{\infty} \frac{\partial P(H|q^S, \theta)}{\partial \theta} \phi(q^S) dq^S
    \end{align*}
    Define the ``boost function" as $g(q^S, \theta) = \frac{\partial P(H|q^S, \theta)}{\partial \theta}$. This function represents how much the conditional hiring probability increases for a small increase in correlation $\theta$, given a screening score $q^S$. 
    As shown before (Step 5 of the proof), $g(q^S, \theta)$ increases with $q^S$ (i.e., $\frac{\partial g}{\partial q^S} > 0$).

    Now, compare the integrals for $N'(\theta)$ and $M'(\theta)$:
    \begin{itemize}
        \item $N'(\theta)$ involves integrating the boost function $g(q^S, \theta)$ over the range $q^S \in [\tau_f^S, \infty)$.
        \item $M'(\theta)$ involves integrating the same boost function $g(q^S, \theta)$ over the range $q^S \in [\tau_m^S, \infty)$.
        \item Since $p_a < 0.5$, we have $\tau_f^S < \tau_m^S$. The integration range for men is shifted to include only higher screening scores compared to the range for women.
        \item Because the boost function $g(q^S, \theta)$ is larger for higher $q^S$, the \emph{average value} of the boost function over the men's integration range $[\tau_m^S, \infty)$ will be greater than its average value over the women's integration range $[\tau_f^S, \infty)$. Let $\bar{g}_f(\theta)$ and $\bar{g}_m(\theta)$ represent these average boosts for shortlisted women and men, respectively. Then $\bar{g}_f(\theta) < \bar{g}_m(\theta)$.
        \item The overall derivatives $N'(\theta)$ and $M'(\theta)$ are related to these average boosts multiplied by the respective probabilities of being shortlisted. Specifically, the average boost experienced by shortlisted men ($\bar{g}_m$) is higher than that for women ($\bar{g}_f$).
    \end{itemize}
    As $\theta$ increases, the hiring probability for men increases relatively more strongly. That is, the higher average boost $\bar{g}_m$ leads to $\frac{M'(\theta)}{M(\theta)}$ being larger than $\frac{N'(\theta)}{N(\theta)}$.

    Since the men's hiring probability increases relatively faster (or decreases slower) than the women's hiring probability as $\theta$ increases, the proportion of women in the hired pool, $p_h(\theta) = \frac{N(\theta)}{N(\theta) + M(\theta)}$, must decrease. 
    
    \item \textbf{Conclusion.} Because the relative increase in hiring probability with $\theta$ is greater for men than for women (when $p_a < 0.5$ and $\tau_f^S < \tau_m^S$), the proportion of women in the hired pool, $p_h(\theta)$, decreases as $\theta$ increases from 0 towards 1.
\end{enumerate}

\end{proof}

\subsection{Proof for Proposition 2}
\label{sec:proof_delta_vs_ph}
\begin{proposition*}
    \cmdphVsDelta{}
\end{proposition*}
\begin{proof}
    The proof directly comes Plackett's identity, which is used to show a known result that for a bivariate normal distribution with correlation $\rho$, $\Phi_2(x,y;\rho)$ is increasing in $\rho$ (see \textcite{tong2012multivariate}).
    
    The probability that a candidate is hired is: 
    \[P(Q^H > \tau^H, Q^S > \tau^S) = \bPhTwo(\tau^H,\tau^S;\rho)\ = \Phi_2(-\tau^H,-\tau^S;\rho)\]
    where $\bPhTwo(\tau^H,\tau^S;\rho)$ is the complementary joint standard normal CDF of $(Q^S, Q^H)$ with correlation $\rho$.
    By symmetry, we have $\bPhTwo(\tau^H,\tau^S;\rho)\ = \Phi_2(-\tau^H,-\tau^S;\rho)$
    
    Here, $\rho_{male} = \theta$ and $\rho_{female} = \theta-\delta$.
    Since $\theta > \theta-\delta$, the probability that a female is hired decreases with $\delta$.
    Consequently, the proportion of women in the hired pool decreases with $\delta$.
\end{proof}

\subsection{Proof for Proposition 3}
\label{sec:proof_theta_vs_eq}
\begin{proposition*}
    \cmdEQvsTheta{}
\end{proposition*}
\begin{proof}
    The proof is of two parts.
    First, the goal is to compute the derivative of the conditional expectation of hires $E[Q_h] \equiv E[Q | Q^S > \tau^S, Q^H > \tau^H]$ with respect to $\theta$, and show that it decreases in the specified range.
    Second, we show that $E[Q_h]$ is globally maximized at $\theta = 0$
    
    \paragraph{Part 1: Derivative of $E[Q_h]$.} The goal of the first part is to show that the derivative of $E[Q_h]$ wrt $\theta$ is negative whenever $0 \leq \theta \leq \min{\{\frac{\theta^S}{\theta^H}, \frac{\theta^H}{\theta^S}\}}$.
    
    \begin{enumerate}[wide, labelwidth=!, labelindent=0pt, leftmargin=*]
    \item\textbf{Conditional Expectation of Truncated Multi-Normal Distribution.} \textcite{tallis1961moment} shows that the expected value of a truncated multi-normal distribution is given by:
        
        \begin{align}\label{equ:expected_quality_of_hires}
            E\bigl[Q \mid Q^H>\tau^H,Q^S>\tau^S\bigr]
            =
            \frac{
            \theta^H\,\ph(\tau^H)\,\bPh\!\Bigl(\frac{\tau^S-\theta\,\tau^H}{\sqrt{1-\theta^2}}\Bigr)
            \;+\;
            \theta^S\,\ph(\tau^S)\,\bPh\!\Bigl(\frac{\tau^H-\theta\,\tau^S}{\sqrt{1-\theta^2}}\Bigr)
            }{
            \bPh_2(\tau^H,\tau^S;\theta)
            }
        \end{align}
    By symmetry, we have:
    $$
    E[Q_h] = \frac{
        \theta^H \phi(\tau^H) \Phi\left( \frac{\theta \tau^H - \tau^S}{\sqrt{1-\theta^2}} \right) +
    \theta^S \phi(\tau^S) \Phi\left( \frac{\theta \tau^S - \tau^H}{\sqrt{1-\theta^2}} \right) 
    }{
    \Phi_2(-\tau^H, -\tau^S; \theta)
    }
    $$
    Let $N(\theta)$ be the numerator and $D(\theta)$ be the denominator:
    \begin{align}
    N(\theta) &= \theta^S \phi(\tau^S) \Phi\left( \frac{\theta \tau^S - \tau^H}{\sqrt{1-\theta^2}} \right) + \theta^H \phi(\tau^H) \Phi\left( \frac{\theta \tau^H - \tau^S}{\sqrt{1-\theta^2}} \right) \label{eq:N_theta} \\
    D(\theta) &= \Phi_2(-\tau^S, -\tau^H; \theta) \label{eq:D_theta}
    \end{align}
    We will use the quotient rule for differentiation:
    $$
    \frac{\partial E[Q_h]}{\partial \theta} = \frac{\frac{\partial N}{\partial \theta} D(\theta) - N(\theta) \frac{\partial D}{\partial \theta}}{[D(\theta)]^2}
    $$
    
    \item\textbf{Derivative of the Denominator $D(\theta)$.}
    Using Plackett's Identity:
    $$
    \frac{\partial D}{\partial \theta} = \frac{\partial}{\partial \theta} \Phi_2(-\tau^S, -\tau^H; \theta) = \phi_2(-\tau^S, -\tau^H; \theta)
    $$
    Using Identity 2:
    $$
    \frac{\partial D}{\partial \theta} = \phi_2(\tau^S, \tau^H; \theta)
    $$
    Let's denote this as $D'(\theta) = \phi_2(\tau^S, \tau^H; \theta)$.
    
    \item\textbf{Derivative of the Numerator $N(\theta)$.}
    Let the arguments of $\Phi(\cdot)$ in $N(\theta)$ be:
    $$ k_S(\theta) = \frac{\theta \tau^S - \tau^H}{\sqrt{1-\theta^2}} \quad \text{and} \quad k_H(\theta) = \frac{\theta \tau^H - \tau^S}{\sqrt{1-\theta^2}} $$
    We need the derivatives $\frac{dk_S}{d\theta}$ and $\frac{dk_H}{d\theta}$.
    For $k_S(\theta) = (\theta \tau^S - \tau^H)(1-\theta^2)^{-1/2}$:
    \begin{align*}
    \frac{dk_S}{d\theta} &= (\tau^S)(1-\theta^2)^{-1/2} + (\theta \tau^S - \tau^H) \left(-\frac{1}{2}\right)(1-\theta^2)^{-3/2}(-2\theta) \\
    &= \frac{\tau^S(1-\theta^2) + \theta(\theta \tau^S - \tau^H)}{(1-\theta^2)^{3/2}} \\
    &= \frac{\tau^S - \theta^2 \tau^S + \theta^2 \tau^S - \theta \tau^H}{(1-\theta^2)^{3/2}} = \frac{\tau^S - \theta \tau^H}{(1-\theta^2)^{3/2}}
    \end{align*}
    Similarly, for $k_H(\theta) = (\theta \tau^H - \tau^S)(1-\theta^2)^{-1/2}$ (by swapping $\tau^S \leftrightarrow \tau^H$ in the expression for $\frac{dk_S}{d\theta}$):
    $$
    \frac{dk_H}{d\theta} = \frac{\tau^H - \theta \tau^S}{(1-\theta^2)^{3/2}}
    $$
    Now, we differentiate $N(\theta)$ using the chain rule $\frac{d}{d\theta} \Phi(w(\theta)) = \phi(w(\theta)) \frac{dw}{d\theta}$:
    $$
    \frac{\partial N}{\partial \theta} = \theta^S \phi(\tau^S) \left[ \phi(k_S(\theta)) \frac{dk_S}{d\theta} \right] + \theta^H \phi(\tau^H) \left[ \phi(k_H(\theta)) \frac{dk_H}{d\theta} \right]
    $$
    Substitute $\frac{dk_S}{d\theta}$ and $\frac{dk_H}{d\theta}$:
    $$
    \frac{\partial N}{\partial \theta} = \theta^S \phi(\tau^S) \phi\left(\frac{\theta \tau^S - \tau^H}{\sqrt{1-\theta^2}}\right) \frac{\tau^S - \theta \tau^H}{(1-\theta^2)^{3/2}} + \theta^H \phi(\tau^H) \phi\left(\frac{\theta \tau^H - \tau^S}{\sqrt{1-\theta^2}}\right) \frac{\tau^H - \theta \tau^S}{(1-\theta^2)^{3/2}}
    $$
    We use Identity 3. Let $u(\theta) = \frac{\tau^H - \theta \tau^S}{\sqrt{1-\theta^2}}$, so $k_S(\theta) = -u(\theta)$.
    And let $v(\theta) = \frac{\tau^S - \theta \tau^H}{\sqrt{1-\theta^2}}$, so $k_H(\theta) = -v(\theta)$.
    By Identity 4, $\phi(k_S(\theta)) = \phi(-u(\theta)) = \phi(u(\theta))$ and $\phi(k_H(\theta)) = \phi(-v(\theta)) = \phi(v(\theta))$.
    So,
    \begin{align*}
    \phi(\tau^S) \phi(k_S(\theta)) &= \phi(\tau^S) \phi\left(\frac{\tau^H - \theta \tau^S}{\sqrt{1-\theta^2}}\right) = \sqrt{1-\theta^2} \phi_2(\tau^S, \tau^H; \theta) = \sqrt{1-\theta^2} D'(\theta) \\
    \phi(\tau^H) \phi(k_H(\theta)) &= \phi(\tau^H) \phi\left(\frac{\tau^S - \theta \tau^H}{\sqrt{1-\theta^2}}\right) = \sqrt{1-\theta^2} \phi_2(\tau^H, \tau^S; \theta) = \sqrt{1-\theta^2} D'(\theta)
    \end{align*}
    Substituting these into $\frac{\partial N}{\partial \theta}$:
    \begin{align*}
    \frac{\partial N}{\partial \theta} &= \theta^S \left(\sqrt{1-\theta^2} D'(\theta)\right) \frac{\tau^S - \theta \tau^H}{(1-\theta^2)^{3/2}} + \theta^H \left(\sqrt{1-\theta^2} D'(\theta)\right) \frac{\tau^H - \theta \tau^S}{(1-\theta^2)^{3/2}} \\
    &= \frac{D'(\theta)}{1-\theta^2} \left[ \theta^S (\tau^S - \theta \tau^H) + \theta^H (\tau^H - \theta \tau^S) \right] \\
    &= \frac{D'(\theta)}{1-\theta^2} \left[ \theta^S \tau^S - \theta^S \theta \tau^H + \theta^H \tau^H - \theta^H \theta \tau^S \right] \\
    &= \frac{D'(\theta)}{1-\theta^2} \left[ (\theta^S \tau^S + \theta^H \tau^H) - \theta (\theta^S \tau^H + \theta^H \tau^S) \right]
    \end{align*}
    Let this be $N'(\theta)$.
    
    \item\textbf{Assembling the Derivative.}
    Using the quotient rule formulation $\frac{\partial E}{\partial \theta} = \frac{N'(\theta)}{D(\theta)} - E \frac{D'(\theta)}{D(\theta)}$:
    \begin{align*}
    \frac{\partial E}{\partial \theta} &= \frac{1}{D(\theta)} \left( \frac{D'(\theta)}{1-\theta^2} \left[ (\theta^S \tau^S + \theta^H \tau^H) - \theta (\theta^S \tau^H + \theta^H \tau^S) \right] \right) - E \frac{D'(\theta)}{D(\theta)} \\
    &= \frac{D'(\theta)}{D(\theta)} \left\{ \frac{(\theta^S \tau^S + \theta^H \tau^H) - \theta (\theta^S \tau^H + \theta^H \tau^S)}{1-\theta^2} - E \right\}
    \end{align*}
    Substituting back the full forms for $D(\theta) = \Phi_2(-\tau^S, -\tau^H; \theta)$, $D'(\theta) = \phi_2(\tau^S, \tau^H; \theta)$:
    $$
    \frac{\partial E[Q_h]}{\partial \theta} = \frac{\phi_2(\tau^S, \tau^H; \theta)}{\Phi_2(-\tau^S, -\tau^H; \theta)} \left[
    \frac{(\theta^S \tau^S + \theta^H \tau^H) - \theta (\theta^S \tau^H + \theta^H \tau^S)}{1-\theta^2} - E[Q_h]
    \right]
    $$
    Note that the first term inside the bracket is the conditional expectation of $E[Q | Q^S = \tau^S, Q^H = \tau^H]$. Rewriting the formula, we get:
    \begin{align*}
        \boxed{
    \frac{\partial E[Q_h]}{\partial \theta} = \frac{\phi_2(\tau^S, \tau^H; \theta)}{\Phi_2(-\tau^S, -\tau^H; \theta)} \left[
    E[Q | Q^S = \tau^S, Q^H = \tau^H] - E[Q | Q^S > \tau^S, Q^H > \tau^H]
    \right]
    }
    \end{align*}
    
    \item\textbf{Sign of the Derivative.}
    Since $\phi_2 > 0$ and $\Phi_2 > 0$, the sign of the derivative is the same as the sign of $E[Q | Q^S = \tau^S, Q^H = \tau^H] - E[Q | Q^S > \tau^S, Q^H > \tau^H]$.
    The derivative is therefore negative whenever $E[Q | Q^S = \tau^S, Q^H = \tau^H] > E[Q | Q^S > \tau^S, Q^H > \tau^H]$ --- i.e., when the corner solution is lower than the average solution in the quadrant above the thresholds.
    
    The corner solution is guaranteed to be lower than the average whenever the conditional expected quality increases with thresholds:
    \begin{align*}
        \frac{\partial}{\partial \tau^S}E[Q | Q^S = \tau^S, Q^H = \tau^H] > 0 \\ 
        \frac{\partial}{\partial \tau^H}E[Q | Q^S = \tau^S, Q^H = \tau^H] > 0
    \end{align*}
    Substituting the full form of $E[Q | Q^S = \tau^S, Q^H = \tau^H]$:
    \begin{align*}
        \frac{\partial}{\partial \tau^S}E[Q | Q^S = \tau^S, Q^H = \tau^H] &= \frac{\partial}{\partial \tau^S} \frac{(\theta^S \tau^S + \theta^H \tau^H) - \theta (\theta^S \tau^H + \theta^H \tau^S)}{1-\theta^2} = \frac{\theta^S - \theta\theta^H}{1-\theta^2} > 0 \\
        \frac{\partial}{\partial \tau^H}E[Q | Q^S = \tau^S, Q^H = \tau^H] &= \frac{\partial}{\partial \tau^H} \frac{(\theta^S \tau^S + \theta^H \tau^H) - \theta (\theta^S \tau^H + \theta^H \tau^S)}{1-\theta^2} = \frac{\theta^H - \theta\theta^S}{1-\theta^2} > 0
    \end{align*}
    This gives a bound for when the derivative is guaranteed to be negative. 
    \begin{align}
        0 \leq \theta \leq \min{\{\frac{\theta^S}{\theta^H}, \frac{\theta^H}{\theta^S}\}}
    \end{align}
    Without loss of generality, we assume that $\theta^S \leq \theta^H$ (i.e., the screener has lower quality than the hiring manager in predicting overall quality), so the condition gets simplified to:
    \begin{align}\label{equ:theta_bound}
        0 \leq \theta \leq \frac{\theta^S}{\theta^H}
    \end{align}
    Intuitively, this is the setting where the screener's correlation to predict the hiring manager's preferences is lower than the screener's relative ability to predict the candidate's actual quality compared to the hiring manager's ability to do so. In other words, we want the screener to be better at predicting the true quality than it is at predicting the hiring manager's preferences.     
    This completes the first part of the proof.    
\end{enumerate}

\paragraph{Part 2: Global Maximum.} The goal of part 2 is to show that $E[Q | Q^S > \tau^S, Q^H > \tau^H]_{\theta=0} > E[Q | Q^S > \tau^S, Q^H > \tau^H]_{1 > \theta > 0}$.

    \begin{enumerate}[wide, labelwidth=!, labelindent=0pt, leftmargin=*]
        \item\textbf{Expected Quality of Hires.} Using \textcite{tallis1961moment} again, we have:
        \begin{align*}
            E\bigl[Q \mid Q^H>\tau^H,Q^S>\tau^S\bigr]
            =
            \frac{
            \theta^H\,\ph(\tau^H)\,\bPh\!\Bigl(\frac{\tau^S-\theta\,\tau^H}{\sqrt{1-\theta^2}}\Bigr)
            \;+\;
            \theta^S\,\ph(\tau^S)\,\bPh\!\Bigl(\frac{\tau^H-\theta\,\tau^S}{\sqrt{1-\theta^2}}\Bigr)
            }{
            \bPh_2(\tau^H,\tau^S;\theta)
            }
        \end{align*}
        \item \textbf{Case when $\theta=0$.}
        When $\theta=0$, $Q^H$ and $Q^S$ are independent. Then,
        \begin{align*}
            \bPh_2(\tau^H,\tau^S;0)
            = \bPh(\tau^H)\,\bPh(\tau^S)
        \end{align*}
        
        This simplifies the above expression to:
        \begin{align}
            E\bigl[q \mid h>\tau^H,s>\tau^S\bigr]_{\theta=0}
            =\frac{
            \theta^H\,\ph(\tau^H)\,\bPh(\tau^S)
            +\theta^S\,\ph(\tau^S)\,\bPh(\tau^H)
            }{
            \bPh(\tau^H)\,\bPh(\tau^S)
            } 
            =\theta^H\,\frac{\ph(\tau^H)}{\bPh(\tau^H)}
            +\theta^S\,\frac{\ph(\tau^S)}{\bPh(\tau^S)}
        \end{align}
        
        \item \textbf{Difference in expected quality \(\Delta(\theta)\).} Define the difference in expected quality as:
        \[
        \Delta(\theta)
        :=E[Q \mid Q^H>\tau^H,Q^S>\tau^S]_{\theta=0}
        \;-\;
        E[Q \mid Q^H>\tau^H,Q^S>\tau^S]_{0< \theta < 1}
        \]
        
        This can be rearranged as:
        \[
        \Delta(\theta)
        =\theta^H\,\ph(\tau^H)\,D_H(\theta)
        \;+\;
        \theta^S\,\ph(\tau^S)\,D_S(\theta),
        \]
        where
        \begin{equation}\label{equ:D_H}
        D_H(\theta)
        =\frac{1}{\bPh(\tau^H)}
        \;-\;
        \frac{
        \bPh\!\bigl(\tfrac{\tau^S-\theta\tau^H}{\sqrt{1-\theta^2}}\bigr)
        }{
        \bPh_2(\tau^H,\tau^S;\theta)
        },
        \end{equation}
        and $D_S(\theta)$ is the analogous term swapping $Q^H\leftrightarrow Q^S$.
        \[
        D_S(\theta)
        =\frac{1}{\bPh(\tau^S)}
        \;-\;
        \frac{
        \bPh\!\bigl(\tfrac{\tau^H-\theta\tau^S}{\sqrt{1-\theta^2}}\bigr)
        }{
        \bPh_2(\tau^H,\tau^S;\theta)
        },
        \]
        
        Because $\theta^H$, $\theta^S$, $\ph(\cdot)$ are all positive, the sign of $\Delta(\theta)$ depends on the sign of $D_H(\theta)$ and $D_S(\theta)$.
        Thus, to prove that $\Delta(\theta) > 0$ for all $\theta \in (0, 1)$, it is sufficient to show that $D_H(\theta) > 0$ and $D_S(\theta) > 0$ for all $\theta \in (0, 1)$.
        
        \item \textbf{Showing $D_H(\theta)>0$ and $D_S(\theta)>0$.} First consider $D_H(\theta)$. Multiplying both sides of \ref{equ:D_H} by $\bPh_2(\tau^H,\tau^S;\theta)$, we get:
        \[
        \bPh_2(\tau^H,\tau^S;\theta)\,D_H(\theta)
        =
        \frac{\bPh_2(\tau^H,\tau^S;\theta)}{\bPh(\tau^H)}
        -\bPh\!\Bigl(\tfrac{\tau^S-\theta\tau^H}{\sqrt{1-\theta^2}}\Bigr).
        \]
        
        Note that:
        \[
        \frac{\bPh_2(\tau^H,\tau^S;\theta)}{\bPh(\tau^H)}
        = P\bigl(Q^S>\tau^S \mid Q^H>\tau^H\bigr),
        \]
        
        For the second term, since in a bivariate normal distribution with correlation $\theta$, the conditional probability of $Q^S \mid Q^H=\tau^H$ is normal with mean $\theta\tau^H$ and variance $1-\theta^2$, we get:
        \[
        \bPh\!\Bigl(\tfrac{\tau^S-\theta\tau^H}{\sqrt{1-\theta^2}}\Bigr)
        = P\bigl(Q^S>\tau^S \mid Q^H=\tau^H\bigr).
        \]
        
        Therefore,  
        \[
            \bPh_2(\tau^H,\tau^S;\theta)\,D_H(\theta) = P\bigl(Q^S>\tau^S \mid Q^H>\tau^H\bigr) - P\bigl(Q^S>\tau^S \mid Q^H=\tau^H\bigr).
        \]
        
        Because $\bPh_2(\cdot)$ is positive, the sign of $D_H(\theta)$ is the same as the sign of the above difference.
        To show that $D_H(\theta)>0$, we need to show that:
        \begin{equation}\label{equ:D_H_inequality}
            P\bigl(Q^S>\tau^S \mid Q^H>\tau^H\bigr) > P\bigl(Q^S>\tau^S \mid Q^H=\tau^H\bigr).
        \end{equation}
        
        Next, define:
        \[
            g(q^H)=P\bigl(Q^S>\tau^S \mid Q^H=q^H\bigr)
            =\bPh\!\Bigl(\tfrac{\tau^S-\theta q^H}{\sqrt{1-\theta^2}}\Bigr),
        \]
        
        Note that the right-hand side term in \ref{equ:D_H_inequality} is the point probability that $Q^S>\tau^S$ given that $Q^H=\tau^H$.
        The left-hand side term is the \emph{average} probability that $Q^S>\tau^S$ given that $Q^H>\tau^H$ over the range of $\tau^H < Q^H < \infty$.
        We can rewrite \ref{equ:D_H_inequality} as:
        \[
            \mathbb{E}\bigl[g(q^H)\mid q^H>\tau^H\bigr] > g(\tau^H).
        \]
        Taking the derivative of $g(q^H)$ with respect to $q^H$, we get:
        \begin{equation}\label{equ:D_H_derivative}
            \frac{d}{d q^H}g(q^H) = \phi\left(\frac{\theta q^H - \tau^S}{\sqrt{1-\theta^2}}\right) \frac{\theta}{\sqrt{1-\theta^2}}
        \end{equation}
        The derivative is always positive since $\phi(\cdot) > 0$ and $\theta > 0$.
        Therefore, $g(q^H)$ is strictly increasing in $q^H$.
        
        Since $g(q^H)$ is strictly increasing, the expected value of $g(q^H)$ over the range of $\tau^H < q^H < \infty$ is greater than $g(\tau^H)$.
        Therefore, $D_H(\theta)>0$.
        A symmetric argument swapping $Q^H \leftrightarrow Q^S$ shows that $D_S(\theta)>0$.
        
        It follows that
    \(\Delta(\theta)>0\) for all $\theta\in(0,1)$.  
    Hence,
    \begin{equation}\label{equ:expected_quality_of_hires_no_constraint}
    E[Q\mid Q^H>\tau^H,Q^S>\tau^S]_{0}
    >
    E[Q\mid Q^H>\tau^H,Q^S>\tau^S]_{\theta}
        \quad
        \forall\,\theta\in(0,1),
        \end{equation}
    \end{enumerate}

    \paragraph{Equal Selection Constraint.} The above proofs do not explicitly consider the equal selection constraint, where the screening threshold $\tau^S$ differs for men ($\tau_m^S$) and women ($\tau_f^S$). 
    The overall expected quality is simply a weighted average based on the proportion $p_h(\theta)$ of each group in the hired pool:
    \[ E[Q_h] = p_h(\theta) \cdot E[Q_{h, \text{female}}] + (1-p_h(\theta)) E[Q_{h, \text{male}}] \] 
    Since~\ref{equ:theta_bound} and~\ref{equ:expected_quality_of_hires_no_constraint} apply to each subgroup (using their specific thresholds), it extends to the overall weighted average, completing the proof. 
\end{proof}

\subsection{Difference in quality between men and women}
\label{apx:model_diff_in_qual}
So far we have considered the case where male and female applicants are equally qualified.
We now consider the case where female applicants can be more/less qualified than men on average.
We parametrize this difference using the location parameter $\alpha$, as follows:

\begin{equation}
    \begin{aligned}
        (Q_m, Q_m^S, Q_m^H) \sim \mathcal{N} \left( \begin{bmatrix} 0 & 0 & 0 \end{bmatrix}, \begin{bmatrix}
                                                                                                     1        & \theta^S & \theta^H \\
                                                                                                     \theta^S & 1        & \theta   \\
                                                                                                     \theta^H & \theta   & 1
                                                                                                 \end{bmatrix} \right) \\
        (Q_f, Q_f^S, Q_f^H) \sim \mathcal{N} \left( \begin{bmatrix} \alpha & \alpha & \alpha \end{bmatrix}, \begin{bmatrix}
                                                                                                                    1        & \theta^S & \theta^H \\
                                                                                                                    \theta^S & 1        & \theta   \\
                                                                                                                    \theta^H & \theta   & 1
                                                                                                                \end{bmatrix} \right)
    \end{aligned}
\end{equation}

A positive $\alpha$ implies that women are more qualified on average than men, and a negative $\alpha$ implies the opposite.

The probability that a candidate is hired is given by:
\begin{align}
    Pr(\text{male is hired}) = \bPh_2(\tau^S_m, \tau^H; 0, \theta) \\
    Pr(\text{female is hired}) = \bPh_2(\tau^S_f, \tau^H; \alpha, \theta)
\end{align}
where $\bPh_2(\cdot;\alpha,\theta)$ is the bivariate tail CDF of $Q^S$ and $Q^H$ with mean $\alpha$ and correlation $\theta$.

The proportion of women in the hired pool is given by:
\begin{align}
    p_h = \frac{p_a \cdot Pr(\text{female is hired})}{p_a \cdot Pr(\text{female is hired}) + (1-p_a) \cdot Pr(\text{male is hired})}
\end{align}

We solve for this numerically, and plot the proportion of women in the hired pool $p_h$ as a function of $\alpha$ and $\theta$ in \Cref{fig:pr_fhire_vs_alpha}.

When women are more qualified than men ($\alpha > 0$), the equal selection constraint becomes redundant.
Higher mean quality compensates for the lower proportion of women in the applicant pool.
This implies that the proportion of women in the hired pool $p_h$ increases with $\alpha$.
Therefore, the equal selection constraint becomes redundant.

When women are less qualified than men ($\alpha < 0$), equal selection becomes even less effective.

\begin{figure}[H] 
    \caption[]{Female proportion of hires ($p_h$) vs. quality difference parameter ($\alpha$)}
    \centering
    \includegraphics[width=\textwidth, keepaspectratio]{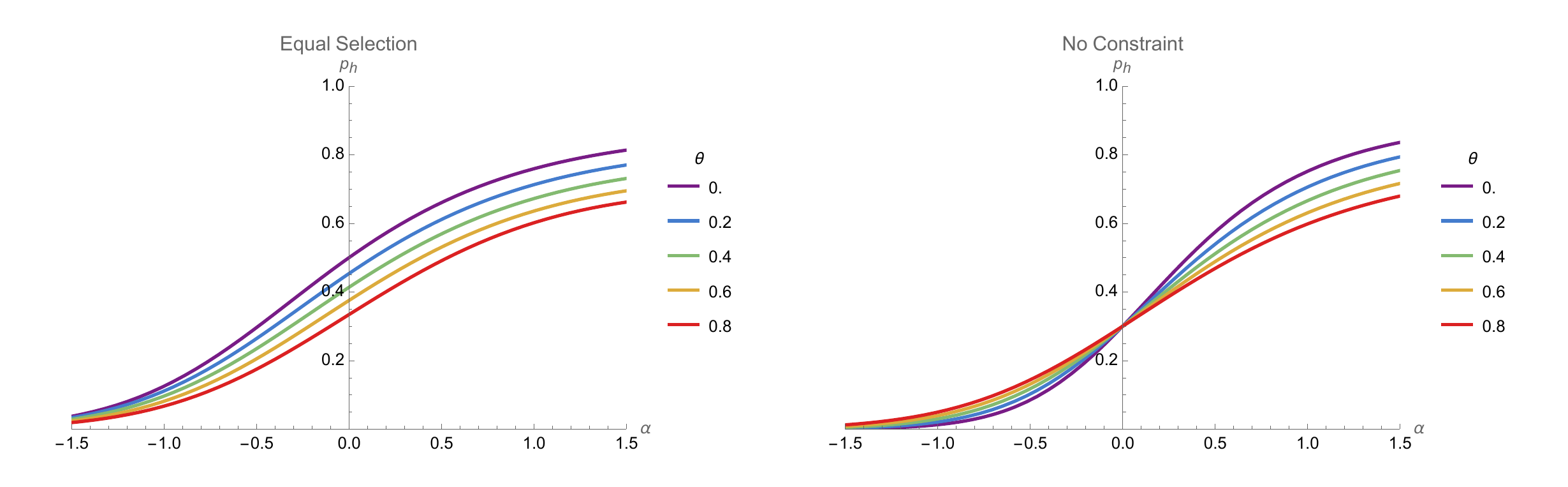}
    \label{fig:pr_fhire_vs_alpha}
    \vspace*{-0.2in}
    \begin{minipage}{1\textwidth}
        \begin{footnotesize}
            \begin{singlespace}
                \emph{Notes:}
                This figure plots the female proportion of hires, $p_h$, as a function of the quality difference parameter, $\alpha$. The proportion of women in the applicant pool is $p_a=0.3$.
                This result does not depend on the $\theta^S$ and $\theta^H$ parameters.
            \end{singlespace}
        \end{footnotesize}
    \end{minipage}
\end{figure}

Interestingly, without the equal selection constraint, the female proportion of hires decreases with $\theta$ when women have higher mean quality ($\alpha > 0$), and vice versa when women have lower mean quality.

\begin{figure}[H] 
    \caption[]{Female proportion of hires ($p_h$) vs. correlation parameter ($\theta$) for different $\alpha$ values}
    \centering
    \includegraphics[width=\textwidth, keepaspectratio]{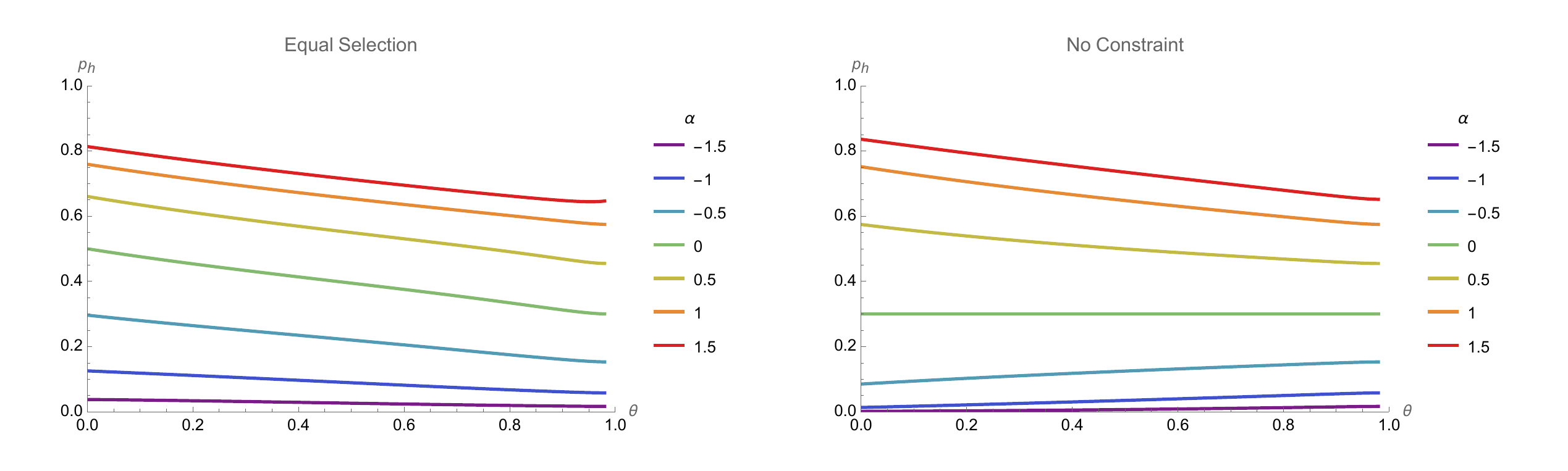}
    \label{fig:ph_vs_theta_alpha}
    \vspace*{-0.2in}
    \begin{minipage}{1\textwidth}
        \begin{footnotesize}
            \begin{singlespace}
                \emph{Notes:}
                This figure plots the female proportion of hires, $p_h$, as a function of the correlation parameter, $\theta$ for different values of $\alpha$. The proportion of women in the applicant pool is $p_a=0.3$.
                This result does not depend on the $\theta^S$ and $\theta^H$ parameters.
            \end{singlespace}
        \end{footnotesize}
    \end{minipage}
\end{figure}

    \section{Additional details on the ML models}
    \label{apx:clf_models}

    \subsection{Predictive performance}

    We measure the predictive performance of the ML models using the Area Under ROC curve (AUC) criteria\footnote{AUC is a widely-used measure for predictive performance for classification models since it is agnostic to both imbalanced classes and classification thresholds.
        The score ranges from 0.5 to 1, where 0.5 corresponds to a random classifier, and 1 corresponds to a perfect classifier.
    } on the hold-out test set and report the results in \Cref{tab:model_performance_gender}.

    For the screening model, the overall AUC score is 0.83, and there is no difference in AUC scores between the male and female candidates.
    We also find that there is some heterogeneity in performance across job types, as reported in \Cref{tab:model_performance_job_cat}.

    For the hiring manager model, the predictive performance is lower compared to the screening model since the hiring manager has more information from the interview, which we do not observe.
    Nonetheless, the predictive performance based on just resume characteristics is still reasonably high at 0.68, and there is no difference between genders.
    Note that the hiring manager model is evaluated on a subset of applicants in the hold-out test set who were, in fact, shortlisted.

    \begin{table}[H] 
        \centering
        \caption{Predictive model performance by gender on hold-out test set}
        \label{tab:model_performance_gender}
        \begin{adjustbox}{max width=\textwidth}
            \begin{threeparttable}
                \begin{tabular}{lrr|rr}
    \toprule
    {}      & \multicolumn{2}{c}{Screening} & \multicolumn{2}{c}{Hiring Manager}                  \\
    Group      & AUC                           & Support                       & AUC  & Support \\
    \midrule
    Female  & 0.83                          & 31,364                        & 0.68 & 4,679   \\
    Male    & 0.83                          & 42,393                        & 0.68 & 6,678   \\
    \midrule
    Overall & 0.83                          & 73,757                        & 0.68 & 11,357  \\
    \bottomrule
\end{tabular}

                \begin{tablenotes}[flushleft]
                    \footnotesize
                    \item \emph{Notes: This table reports the predictive performance of the screening and hiring manager classification models on the hold-out test set broken down by male and female candidates.
                    }
                \end{tablenotes}
            \end{threeparttable}
        \end{adjustbox}
    \end{table}

    \begin{table}[H] 
        \centering
        \caption{Predictive model performance by job category on hold-out test set}
        \label{tab:model_performance_job_cat}
        \begin{adjustbox}{max width=\textwidth}
            \begin{threeparttable}
                \begin{tabular}{lrr|rr}
    \toprule
    {}                                  & \multicolumn{2}{c}{Screening} & \multicolumn{2}{c}{Hiring Manager}                  \\
    Job Category                        & AUC                           & Support                       & AUC  & Support \\
    \midrule
    Legal \& PR                         & 0.86                          & 7,337                         & 0.65 & 926     \\
    Product \& Design                   & 0.85                          & 12,519                        & 0.66 & 1,863   \\
    Sales \& Marketing                  & 0.85                          & 15,169                        & 0.67 & 2,120   \\
    Other                               & 0.83                          & 214                           & 0.59 & 25      \\
    Engineering \& Technical            & 0.82                          & 16,506                        & 0.66 & 3,315   \\
    Finance \& Accounting               & 0.82                          & 7,026                         & 0.67 & 886     \\
    Biz Dev \& Operations               & 0.81                          & 4,836                         & 0.65 & 628     \\
    HR                                  & 0.79                          & 3,355                         & 0.69 & 452     \\
    Customer Service \& Acct Management & 0.78                          & 6,734                         & 0.7  & 1,112   \\
    \midrule
    Overall                             & 0.83                          & 73,757                        & 0.68 & 11,357  \\
    \bottomrule
\end{tabular}

                \begin{tablenotes}[flushleft]
                    \footnotesize
                    \item \emph{Notes: This table reports the predictive performance of the screening and hiring manager classification models on the hold-out test set. We estimate the metrics at the job posting level and aggregate up to the job category level.
                    }
                \end{tablenotes}
            \end{threeparttable}
        \end{adjustbox}
    \end{table}

    \begin{figure}[H]
        \centering
        \caption{ROC Curves for Screening and Hiring Manager Models}
        \begin{subfigure}[b]{0.48\textwidth}
            \caption{Screening Model ROC Curve}
            \centering
            \includegraphics[width=\textwidth, keepaspectratio]{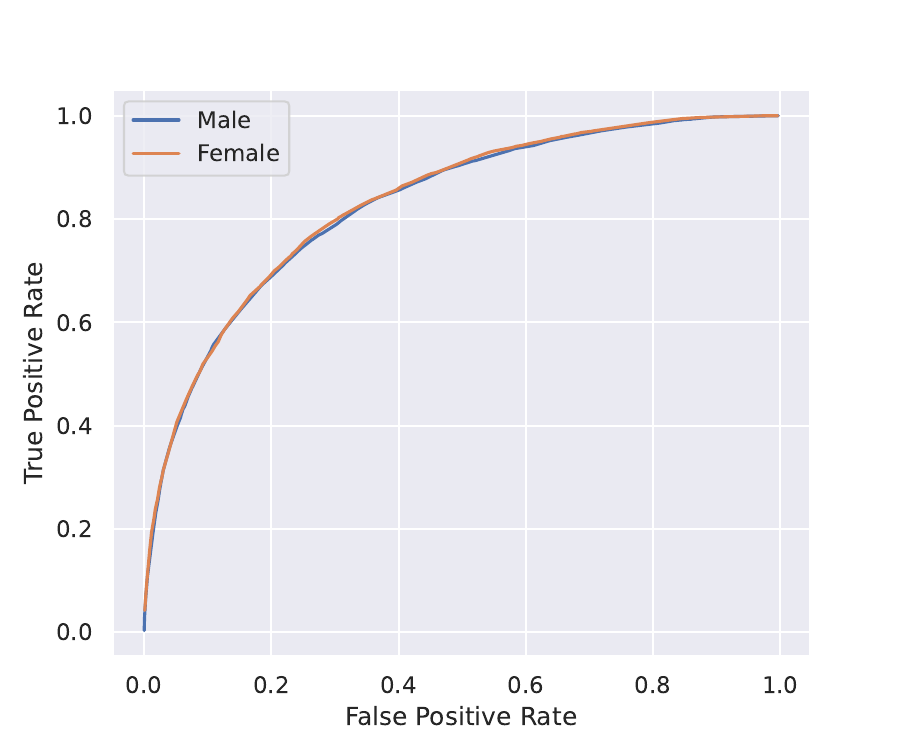}
            \label{fig:screening-roc}
        \end{subfigure}
        \hfill
        \begin{subfigure}[b]{0.48\textwidth}
            \caption{Hiring Manager Model ROC Curve}
            \centering
            \includegraphics[width=\textwidth, keepaspectratio]{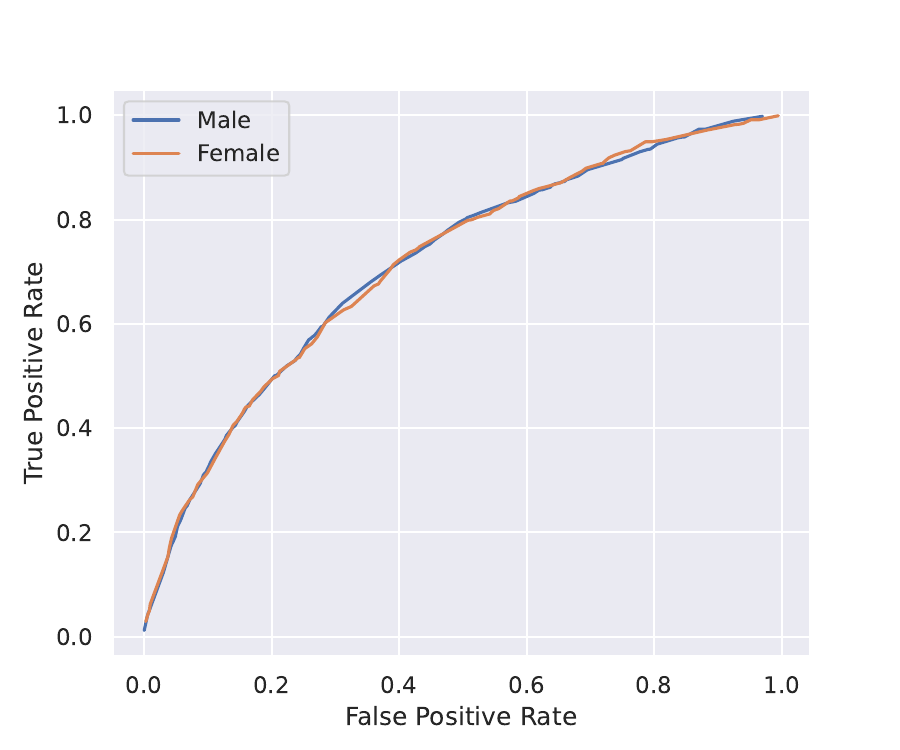}
            \label{fig:hiring-roc}
        \end{subfigure}
        \label{fig:roc-curves}
    \end{figure}

    \section{Additional empirical analyses}
    \label{apx:additional_analyses}

    \subsection{Measures of observable quality differences between men and women}
    \label{apx:emp_diff_in_quality}

    In this section we empirically assess differences in observable quality measures between men and women.
    To do so, we first identify four measures of observable quality: job-resume skill similarity, years of experience, attended a top 100 school, and educational attainment.
    We operationalize these measures as follows:

    \begin{itemize}
        \item \textbf{Job-Resume skill similarity:} We measure the average cosine similarity between skills listed in the job description and the applicant's resume. To get the cosine similarity, we first tokenize the job description and resume text. We then filter the tokens to extract only skills-related tokens (e.g., \texttt{python}, \texttt{data\_analysis}, \texttt{project\_management}) using a dictionary of skills\footnote{
                  This dictionary was created using the skills section of LinkedIn profiles in a separate analysis.}.
              We then get the vector representation of each skill token using a custom word2vec model trained on resumes, and take the average cosine similarity between the job description and resume skill vectors.
        \item \textbf{Years of experience:} We get the applicant's years of experience from the ATS.
        \item \textbf{Attended a top 100 school:} We create a binary variable indicating if the applicant attended a top 100 school based on the undergraduate institution listed in the resume. We use U.S. News and World Report's ranking of top 100 schools as the reference.
        \item \textbf{Educational attainment:} We create binary variables indicating if the applicant has a bachelor's, master's, or doctorate degree based on the highest degree listed in the resume.
    \end{itemize}

    For each of these outcomes, we estimate a linear regression model with job posting fixed effects

    $$y_{ij} = \beta_{\textit{Female}} \cdot \textit{Female}_{i} + \alpha_{j} + \epsilon_{ij}$$

    where $y_{ij}$ is the observable quality measure for applicant $i$ applying to job $j$, $\textit{Female}_{i}$ is a binary variable indicating if the applicant is female, $\alpha_{j}$ is the job posting fixed effect, and $\epsilon_{ij}$ is the error term.

    We report the coefficients and percentage differences below.
    Compared to male applicants, female applicants have roughly the same job-resume skill similarity, fewer years of experience, are more likely to have attended a top 100 school, more likely to have a bachelor's or master's degree, and less likely to have a doctorate degree.

    \begin{table}[H] 
        \centering
        \caption{Regression coefficients of observable quality measures}
        \begin{threeparttable}
            \begin{tabular}{lrrr}
                \toprule
                Variable                    & $\beta_{\textit{Female}}$ & \% Difference         & $p$-value           \\
                \midrule
                Job-Resume skill similarity & \jobResumeSimCoef{}       & \jobResumeSimDiff{}\% & \jobResumeSimPval{} \\
                Yrs exp                     & \yrsExpCoef{}             & \yrsExpDiff{}\%       & \yrsExpPval{}       \\
                Attended top 100 school     & \topschoolCoef{}          & \topschoolDiff{}\%    & \topschoolPval{}    \\
                Has bachelor's degree       & \bachelorsCoef{}          & \bachelorsDiff{}\%    & \bachelorsPval{}    \\
                Has master's degree         & \mastersCoef{}            & \mastersDiff{}\%      & \mastersPval{}      \\
                Has doctorate               & \doctorateCoef            & \doctorateDiff{}\%    & \doctoratePval{}    \\
                \bottomrule
            \end{tabular}
            \begin{tablenotes}[flushleft]
                \footnotesize
                \item \emph{Notes: This table shows the estimated regression coefficients and percentage differences of various observable quality measures. Job-Resume Similarity is the cosine similarity between the job description and the resume. Yrs Exp is the number of years of experience. Attended Top 100 School is a binary variable indicating if the candidate attended a top 100 school. Has Bachelor's Degree, Has Master's Degree, and Has Doctorate are binary variables indicating if the candidate has a bachelor's, master's, or doctorate degree, respectively.}
            \end{tablenotes}
        \end{threeparttable}
    \end{table}

    \subsection{Regression estimates on the likelihood of being shortlisted}
    \label{apx:callback_regression}

    \begin{table}[H] 
        \centering
        \caption{Likelihood of being shortlisted, OLS estimates}
        \begin{threeparttable}
            \begingroup
\centering
\begin{tabular}{lc}
   \tabularnewline \midrule \midrule
   Dependent Variable:   & Shortlisted (1=YES)\\  
   \midrule
   \emph{Variables}\\
   Male            & -0.0128$^{***}$\\   
                         & (0.0012)\\   
   Yrs Exp               & 0.0012$^{***}$\\   
                         & (0.0002)\\   
   Job Resume Similarity & 0.3918$^{***}$\\   
                         & (0.0107)\\   
   \midrule
   \emph{Fixed-effects}\\
   Job Posting           & Yes\\  
   School Rank           & Yes\\  
   Degree                & Yes\\  
   \midrule
   \emph{Fit statistics}\\
   Observations          & 595,246\\  
   \midrule \midrule
   \multicolumn{2}{l}{\emph{Clustered (Job Posting) standard-errors in parentheses}}\\
   \multicolumn{2}{l}{\emph{Signif. Codes: ***: 0.01, **: 0.05, *: 0.1}}\\
\end{tabular}
\par\endgroup

            \begin{tablenotes}[flushleft]
                \footnotesize
                \item \emph{Notes:} This table shows the OLS estimates of the likelihood of being shortlisted. Each observation corresponds to an application. Male applicants are more likely to be shortlisted than female applicants after controlling for job-resume skill similarity, years of experience, education, and job posting.
            \end{tablenotes}
        \end{threeparttable}
    \end{table}

    \subsection{Goodness of fit using different copulas}
    \label{apx:copula_gof}

    Below we provide Kolmogorov-Smirnov (KS) statistics of empirical quality scores $q^S$ and $q^H$ fit against commonly used copulas.
Lower KS statistics indicate a better fit.
The Gaussian copula has the 2nd best fit after the Frank copula.

\begin{table}[H] 
    \centering
    \caption{Copula goodness-of-fit measures}
    \begin{threeparttable}
        \begin{tabular}{lrr}           
    \toprule
    Copula   & KS-Statistic & $p$-value \\ 
    \midrule 
    Gaussian & 2.62         & $<0.0001$ \\
    Gumbel   & 6.25         & $<0.0001$ \\
    Frank    & 2.29         & $<0.0001$ \\
    Clayton  & 6.07         & $<0.0001$ \\
    Joe      & 11.84        & $<0.0001$ \\
    AMH      & 5.37         & $<0.0001$ \\
    \bottomrule
\end{tabular}

        \begin{tablenotes}[flushleft]
            \footnotesize
            \item \emph{Notes: This table shows the goodness-of-fit measures of empirical quality scores $q^S$ and $q^H$ fit against various copulas.}
        \end{tablenotes}
    \end{threeparttable}
\end{table}

\end{document}